\documentclass{article} 
\usepackage{iclr2023_conference,times}

\usepackage{amsmath,amsfonts,bm}

\def\1{\bm{1}}

\DeclareMathAlphabet{\mathsfit}{\encodingdefault}{\sfdefault}{m}{sl}
\SetMathAlphabet{\mathsfit}{bold}{\encodingdefault}{\sfdefault}{bx}{n}

\usepackage{hyperref}
\usepackage{url}

\usepackage[utf8]{inputenc} 
\usepackage[T1]{fontenc}    
\usepackage[utf8]{inputenc} 
\usepackage[T1]{fontenc}    
\usepackage{booktabs}       
\usepackage{amsfonts}       
\usepackage{nicefrac}       
\usepackage{microtype}      
\usepackage{xcolor}         
\usepackage{subcaption}
\usepackage{natbib}
\usepackage{algorithm}
\usepackage{algpseudocode}

\usepackage{amsmath, amsfonts, amssymb, amsthm, mathtools}
\usepackage{bm}

\newtheorem{theorem}{Theorem}
\newtheorem{definition}{Definition}

\newtheorem{lemma}{Lemma}
\newtheorem{assumption}{Assumption}
\newtheorem{innercustomthm}{Algorithm}
\newenvironment{algo}[1]
  {\renewcommand\theinnercustomthm{#1}\innercustomthm}
  {\endinnercustomthm}
\newtheorem{corollary}{Corollary}

\newcommand{\norm}[1]{\left\lVert#1\right\rVert}
\renewcommand{\vec}[1]{\bm{\mathrm{#1}}}
\renewcommand{\matrix}[1]{\ensuremath{\mathbf{#1}}}

\newcommand{\xxi}{\vec{\xi}}
\newcommand{\y}{\vec{y}}

\newcommand{\w}{\vec{w}}
\newcommand{\wstar}{\vec{w}^\star}
\newcommand{\X}{\matrix{X}}
\newcommand{\Xt}{\matrix{X}^\top}

\newcommand{\R}{\mathbb{R}}
\newcommand{\D}{\matrix{D}}

\usepackage{pifont}
\newcommand{\cmark}{\ding{51}}
\newcommand{\xmark}{\ding{55}}

\newcommand{\rowname}[1]
{\rotatebox{90}{\makebox[.4in][c]{#1}}}

\title{Implicit Regularization for Group Sparsity}

\author{
  Jiangyuan Li$^\star$, 
  Thanh V. Nguyen,
  Chinmay Hegde$^\dagger$
  \& 
  Raymond K. W. Wong$^\star$\\
  $^\star$Texas A\&M University\\
  $^\dagger$New York University\\
  \texttt{\{jiangyuanli, raywong\}@tamu.edu};\\
  \texttt{thanhng.cs@gmail.com};
  \texttt{chinmay.h@nyu.edu}\\
}

\iclrfinalcopy
\begin{document}

\maketitle

\begin{abstract}
We study the implicit regularization of gradient descent towards structured sparsity via a novel neural reparameterization, which we call a \emph{``diagonally grouped linear neural network''}. We show the following intriguing property of our reparameterization: gradient descent over the squared regression loss, without any explicit regularization, biases towards solutions with a group sparsity structure. In contrast to many existing works in understanding implicit regularization, we prove that our training trajectory cannot be simulated by mirror descent.
We analyze the gradient dynamics of the corresponding regression problem in the general noise setting and obtain minimax-optimal error rates. Compared to existing bounds for implicit sparse regularization using diagonal linear networks, our analysis with the new reparameterization shows improved sample complexity. In the degenerate case of size-one groups, our approach gives rise to a new algorithm for sparse linear regression. Finally, we demonstrate the efficacy of our approach with several numerical experiments\footnote{Code is available on \url{https://github.com/jiangyuan2li/Implicit-Group-Sparsity}}.
\end{abstract}


\section{Introduction}


\textbf{Motivation.} A salient feature of modern deep neural networks is that they are highly overparameterized with many more parameters than available training examples.
Surprisingly, however, deep neural networks trained with gradient descent can generalize quite well in practice, even without explicit regularization. One hypothesis is that the dynamics of gradient descent-based training itself induce some form of implicit regularization, biasing toward solutions with low-complexity~\citep{hardt2016train,neyshabur2017geometry}. Recent research in deep learning theory has validated the hypothesis of such implicit regularization effects. A large body of work, which we survey below, has considered certain (restricted) families of linear neural networks and established two types of implicit regularization --- standard sparse regularization and $\ell_2$-norm regularization --- depending on how gradient descent is initialized. 

On the other hand, the role of \emph{network architecture}, or the way the model is parameterized in implicit regularization, is less well-understood. 
Does there exist a parameterization that promotes implicit regularization of gradient descent towards richer structures beyond standard sparsity? 

In this paper, we analyze a simple, prototypical hierarchical architecture for which gradient descent induces \textit{group} sparse regularization. Our finding --- that finer, \textit{structured} biases can be induced via gradient dynamics --- highlights the richness of co-designing neural networks along with optimization methods for producing more sophisticated regularization effects.

\begin{table}[ht!]
\label{tbl:comparison}
\centering
\begin{tabular}{c|c|c|c|c} 
  & NNs & Noise & Implicit vs. Explicit & Regularization\\ 
 \toprule
\citet{vaskevicius2019implicit} 
& DLNN & \cmark & Implicit (GD) & Sparsity\\
\citet{dai2021representation} & LNN & \xmark & Explicit ($\ell_2$-penalty) &  (Group) Quasi-norm \\
\citet{jagadeesan2021inductive} &
LCNN & \xmark & Explicit ($\ell_2$-penalty) & Norm induced by SDP\\
\citet{wu2020implicit} & DLNN & \xmark & Implicit & $\ell_2$-norm \\
This paper & DGLNN & \cmark & Implicit (GD) & Structured sparsity \\
 \bottomrule
\end{tabular}%
\caption{Comparisons to related work on implicit and explicit regularization. Here, GD stands for gradient descent, (D)LNN/CNN for (diagonal) linear/convolutional neural network, and DGLNN for diagonally grouped linear neural network.}
\end{table}

\textbf{Background.} Many recent theoretical efforts have revisited traditional, well-understood problems such as linear regression~\citep{vaskevicius2019implicit, li2021implicit, zhao2019implicit}, matrix factorization~\citep{gunasekar2018implicit, li2018algorithmic, arora2019implicit} and tensor decomposition~\citep{ge2017learning, wang2020beyond}, from the perspective of neural network training. For nonlinear models with squared error loss, \citet{williams2019gradient} and \citet{jin2020implicit} study the implicit bias of gradient descent in wide depth-2 ReLU networks with input dimension 1. Other works \citep{gunasekar2018implicitb, soudry2018implicit, nacson2019convergence} show that gradient descent biases the solution towards the max-margin (or minimum $\ell_2$-norm) solutions over separable data.

Outside of implicit regularization, several other works study the inductive bias of network architectures under \textit{explicit} $\ell_2$ regularization on model weights~\citep{pilanci2020neural, sahiner2020vector}. For multichannel linear convolutional networks, \citet{jagadeesan2021inductive} show that $\ell_2$-norm minimization of weights leads to a norm regularizer on predictors, where the norm is given by a semidefinite program (SDP).
The representation cost in predictor space induced by explicit $\ell_2$ regularization on (various different versions of) linear neural networks is studied in \citet{dai2021representation},
which demonstrates several interesting (induced) regularizers on the linear predictors
such as $\ell_p$ quasi-norms and group quasi-norms. However, these results are silent on the behavior of gradient descent-based training \emph{without} explicit regularization.
In light of the above results, we ask the following question:

\begin{quote}
Beyond $\ell_2$-norm, sparsity and low-rankness, can gradient descent induce other forms of implicit regularization? 
\end{quote}

\textbf{Our contributions.} In this paper, we rigorously show that a \textit{diagonally-grouped linear neural
network} (see Figure \ref{fig:gdlnn}) trained by gradient descent with (proper/partial) weight normalization induces \emph{group-sparse} regularization: a form of structured  regularization that, to the best of our knowledge, has not been provably established in previous work.

One major approach to understanding implicit regularization of gradient descent is based on its equivalence to a mirror descent (on a different objective function) \citep[e.g.,][]{gunasekar2018characterizing, woodworth2020kernelregimes}.
However, we show that, for the diagonally-grouped linear network architecture, the gradient dynamics is beyond mirror descent.
We then analyze the convergence of gradient flow with early stopping under orthogonal design with possibly noisy observations, and show that the obtained solution exhibits an implicit regularization effect towards structured (specifically, group) sparsity.
In addition,
we show that weight normalization can deal with instability related to the choices of learning rates and initialization.
With weight normalization,
we are able to obtain a similar implicit regularization result but in more general settings: orthogonal/non-orthogonal designs with possibly noisy observations.
Also, the obtained solution can achieve minimax-optimal error rates.

Overall, compared to existing analysis of diagonal linear networks, our model design --- that induces structured sparsity --- exhibits provably improved sample complexity. In the degenerate case of size-one groups, our bounds coincide with previous results, and our approach can be interpreted as a new algorithm for sparse linear regression.

\begin{figure}[ht!]
\centering
\begin{subfigure}{.5\linewidth}
  \centering
  \includegraphics[width=.5\linewidth]{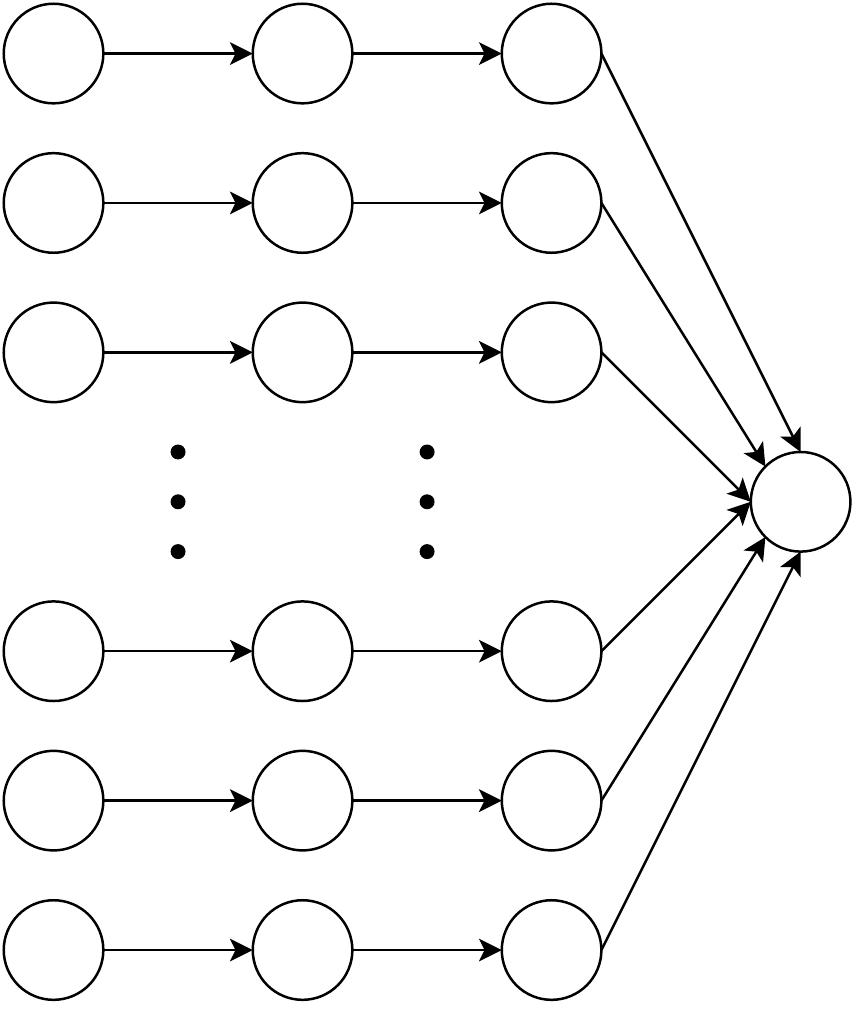}
  \caption{Diagonal linear NN (DLNN).}
  \label{fig:dlnn}
\end{subfigure}%
\begin{subfigure}{.5\linewidth}
  \centering
  \includegraphics[width=.5\linewidth]{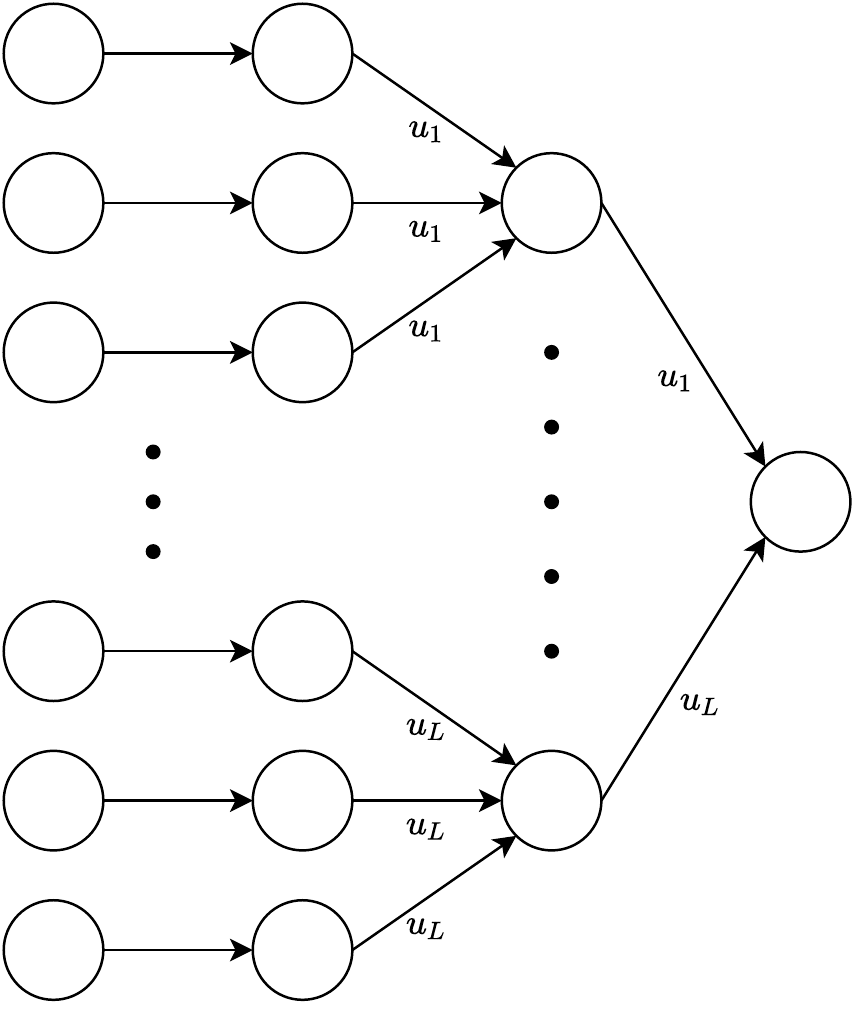}
  \caption{Diagonally grouped linear NN (DGLNN).}
  \label{fig:gdlnn}
\end{subfigure}
\caption{An illustration of the two architectures for standard and group sparse regularization.}
\label{fig:architecture}
\end{figure}

\textbf{Our techniques.} Our approach is built upon the \emph{power reparameterization} trick, which has been shown to promote model sparsity
\citep{schwarz2021powerpropagation}.
Raising the parameters of a linear model element-wisely to the $N$-th power ($N > 1$) results in that parameters of smaller magnitude receive smaller gradient updates, while parameters of larger magnitude receive larger updates. In essence, this leads to a ``rich get richer'' phenomenon in gradient-based training.  In \citet{gissin2019implicit} and \citet{berthier2022incremental}, the authors analyze the gradient dynamics on a toy example, and call this ``incremental learning''. Concretely, for a linear predictor $\vec{w}\in\mathbb{R}^p$, if we re-parameterize the model as $\vec{w} = \vec{u}^{\circ N} - \vec{v}^{\circ N}$ (where $\vec{u}^{\circ N}$ means the $N$-th element-wise power of $\vec{u}$), 
then gradient descent will bias the training towards sparse solutions. This reparameterization is equivalent to a diagonal linear network, as shown in Figure~\ref{fig:dlnn}. This is further studied in \citet{woodworth2020kernelregimes} for interpolating predictors, where they show that a small enough initialization induces $\ell_1$-norm regularization. For noisy settings, \citet{vaskevicius2019implicit} and \citet{li2021implicit} show that gradient descent converges to sparse models with early stopping. In the special case of sparse recovery from under-sampled observations (or compressive sensing), the optimal sample complexity can also be obtained via this reparameterization \citep{chou2021more}. 

Inspired by this approach, we study a novel model reparameterization of the form $\vec{w} = [\vec{w}_1, \ldots, \vec{w}_L]$, where $\vec{w}_l= u_l^2 \vec{v}_l$ for each group $l\in \{1,\ldots, L\}$. (One way to interpret this model is to think of $u_l$ as the ``magnitude'' and $\vec{v}_l$ as the ``direction'' of the subvector corresponding to each group; see Section \ref{sec:setup} for details.) This corresponds to a special type of linear neural network architecture, as shown in Figure~\ref{fig:gdlnn}. A related architecture has also been recently studied in \citet{dai2021representation}, but there the authors have focused on the bias induced by an \textit{explicit} $\ell_2$ regularization on the weights
and have not investigated the effect of gradient dynamics.

 



The diagonally linear network parameterization of \citet{woodworth2020kernelregimes, li2021implicit} does not suffer from identifiability issues. In contrast to that, in our setup the ``magnitude'' parameter $u_l$ of each group interacts with the norm of the ``direction'', $\norm{\vec{v}_l}_2$, causing a fundamental problem of identifiability. By leveraging the layer balancing effect \citep{du2018algorithmic} in DGLNN, we verify the group regularization effect implicit in gradient flow with early stopping. 
But gradient flow is idealized; for a more practical algorithm, we use a variant of gradient descent based on 
\emph{weight normalization}, proposed in \citep{salimans2016weight}, and studied in more detail in \citep{wu2020implicit}. Weight normalization has been shown to be particularly helpful in stabilizing the effect of learning rates \citep{morwani2022inductive, van2017l2}. 
With weight normalization, the learning effect is separated into magnitudes and directions. We derive the gradient dynamics on both magnitudes and directions with perturbations. Directions guide magnitude to grow, and as the magnitude grows, the directions get more accurate. Thereby, we are able to establish regularization effect implied by such gradient dynamics. 


\textbf{A remark on grouped architectures.} Finally, we remark that grouping layers have been commonly used in grouped CNN and grouped attention mechanisms~\citep{xie2017aggregated, wu2021convolutional}, which leads to parameter efficiency and better accuracy. Group sparsity is also useful for deep learning models in multi-omics data for survival prediction~\citep{xie2019group}. We hope our analysis towards diagonally grouped linear NN could lead to more understanding of the inductive biases of grouping-style  architectures.

\section{Setup}
\label{sec:setup}

\textbf{Notation.} Denotes the set $\{1,2,\ldots, L\}$ by $[L]$, and the vector $\ell_2$ norm by $\norm{\cdot}$. We use $\vec{1}_p$ and  $\vec{0}_p$ to denote $p$-dimensional vectors of all 1s and all 0s correspondingly.
Also, $\odot$ represents the entry-wise multiplication whereas $\vec{\beta}^{\circ N}$ denotes element-wise power $N$ of a vector $\vec{\beta}$. We use $\vec{e}_i$ to denote the $i^{\textrm{th}}$ canonical vector. We write inequalities up to multiplicative constants using the notation $\lesssim$, whereby the constants do not depend on any problem parameter. 

\textbf{Observation model.} Suppose that the index set $[p] = \cup_{j=l}^L G_l$ is partitioned into $L$ disjoint (i.e., non-overlapping) groups $G_1,G_2,\ldots, G_L$ where $G_i\cap G_j =\emptyset, \forall i\neq j$. The size of $G_{l}$ is denoted by $p_l = |G_l|$ for $l \in[L]$. Let $\wstar\in\mathbb{R}^p$ be a $p$-dimensional vector where the entries of $\wstar$ are non-zero only on a subset of groups. 
We posit a linear model of data where  observations $(\vec{x}_i, y_i)\in\R^p\times\R,~i \in [n]$ are given such that $y_i=\langle \vec{x}_i,\wstar\rangle + \xi_i$ for $i=1,\ldots,n$, and $\vec{\xi} = [\xi_1, \ldots, \xi_n]^\top$ is a noise vector. Note that we do not impose any special restriction between $n$ (the number of observations) and $p$ (the dimension).
We write the linear model in the following matrix-vector form:
$
\y = \X \wstar + \vec{\xi},
$
with the $n\times p$ design matrix $\X = [\X_1,\X_2,\ldots, \X_L]$, where $\X_l\in\R^{n\times p_l}$
represents the features from the $l^\textrm{th}$ group $G_l$,
for $l\in[L]$.
We make the following assumptions on $\matrix{X}$:

\begin{assumption}
The design matrix $\X$ satisfies 
  \begin{equation}
   \sup_{\norm{\vec{\beta}_1}\leq 1,  \norm{\vec{\beta}_2} \leq 1}\left|
  \left\langle \vec{\beta}_1, \left(\frac1{n}\X_{l}^\top\X_{l} - \matrix{I} \right)\vec{\beta}_2 
  \right\rangle \right|
  \leq 
  \delta_{in}, \quad \text{where } \vec{\beta}_1,\vec{\beta}_2\in\mathbb{R}^{p_l}
  \label{eq:in},
  \end{equation}
  and
  \begin{equation}
  \sup_{\norm{\vec{\beta}_1}\leq 1,  \norm{\vec{\beta}_2} \leq 1} \left|
  \left\langle \frac1{\sqrt{n}}\X_l\vec{\beta}_1, \frac1{\sqrt{n}}\X_{l'}\vec{\beta}_2 
  \right\rangle \right|
  \leq 
  \delta_{out}, \quad \text{where } \vec{\beta}_1\in\mathbb{R}^{p_l},\vec{\beta}_2\in\mathbb{R}^{p_{l'}}, l\neq l',
  \label{eq:out}
  \end{equation}
for some constants $\delta_{in},\delta_{out} \in (0,1)$.
\label{assumptions}
\end{assumption}

The first part \eqref{eq:in} is a within-group eigenvalue condition while the second part \eqref{eq:out} is a between-group block coherence assumption. There are
multiple ways to construct a sensing matrix to fulfill these two conditions~\citep{eldar2009block, baraniuk2010model}. One of them is based on the
fact that random Gaussian matrices satisfy such conditions with high probability \citep{stojnic2009reconstruction}. 


\textbf{Reparameterization.} Our goal is to learn a parameter $\w$ from the data $\{(\vec{x}_i, y_i)\}_{i=1}^n$ with coefficients which obey group structure. Instead of imposing an explicit group-sparsity constraint on $\w$ (e.g., via weight penalization by group), we show that gradient descent on the \emph{unconstrained} regression loss can still learn $\wstar$, provided we design a special reparameterization. Define a mapping $g(\cdot): [p] \to [L]$ from each index $i$ to its group $g(i)$. Each parameter is rewritten as $\quad w_i = u_{g(i)}^2 v_i, \forall i\in[p].$
The parameterization $G(\cdot): \mathbb{R}^L_+\times \mathbb{R}^p \to \mathbb{R}^p$ reads
\begin{align*}
    [u_1,\ldots, u_L, v_1, v_2, \ldots, v_p]
    \to
    [{u_1^2}v_1,{u_1^2}v_2, \ldots, u_L^2 v_p].
\end{align*}

This corresponds to the 2-layer neural network architecture displayed in Figure \ref{fig:gdlnn}, in which $\matrix{W}_1 = \mathrm{diag}(v_{1},\ldots, v_{p})$, and $\matrix{W}_2$ is ``diagonally'' tied  within each group: 
$$  \matrix{W}_2 = \mathrm{diag}(u_{1},\ldots, u_{1},
u_{2},\ldots, u_{2},\ldots,
u_{L},\ldots, u_{L}).$$ 



\textbf{Gradient dynamics.} We learn $\vec{u}$ and $\vec{v}$ by minimizing the standard squared loss:
\[
\mathcal{L}(\vec{u},\vec{v}) = \frac{1}{2} \norm{\y - \X [(\D\vec{u})^{\circ2}\odot \vec{v}] }^2,
\]
where 
\[
\D = 
\begin{pmatrix}
\vec{1}_{p_1} & \vec{0}_{p_1} & \ldots & \vec{0}_{p_1} \\
\vec{0}_{p_2} & \vec{1}_{p_2} & \ldots & \vec{0}_{p_2} \\
\vdots & \vdots & \vdots & \vdots  \\
\vec{0}_{p_L} & \vec{0}_{p_L} & \ldots & \vec{1}_{p_L} \\
\end{pmatrix}
\in\mathbb{R}^{p\times L}.
\]

By simple algebra, the gradients with respect to $\vec{u}$ and $\vec{v}$ read as follows:
\begin{align*}
    \nabla_{\vec{u}} L &= 2\D^\top \left(\vec{v}\odot\left[\X^\top \X ((\D \vec{u})^{\circ 2} \odot\vec{v}-\wstar)
    -  \X^\top \vec{\xi} \right] \odot \D\vec{u}\right), \\
    \nabla_{\vec{v}} L &= \left[\X^\top \X ((\D \vec{u})^{\circ 2} \odot\vec{v}-\wstar)
    -  \X^\top \vec{\xi} \right] \odot (\D \vec{u})^{\circ 2}.\\
\end{align*}
%
%
Denote $\vec{r}(t)= \vec{y} - \sum_{l'=1}^L u_{l}^2(t) \X_{l} \vec{v}_{l}(t)$.  For each group $l\in[L]$, the gradient flow reads
\begin{align}
    \frac{\partial u_l(t)}{\partial t}
    = \frac2n u_l(t) \vec{v}_l^\top(t) \X_l^\top \vec{r}(t),
    \quad
    \frac{\partial \vec{v}_l(t)}{\partial t}
    = \frac1n u_l^2(t) \X_l^\top\vec{r}(t).
    \label{eq:grad-on-each-group}
\end{align}
Although we are not able to transform the gradient dynamics back onto $\vec{w}(t)$ due to the overparameterization, the extra term $u_l(t)$ on group magnitude leads to ``incremental learning'' effect.

\section{Analysis of Gradient Flow}

\subsection{First Attempt: Mirror Flow}

Existing results about implicit bias in overparameterized models are mostly based on recasting the training process from the parameter space $\{\vec{u}(t),\vec{v}(t)\}_{t\geq0}$ to the predictor space $\{\vec{w}(t)\}_{t\geq 0}$ \citep{woodworth2020kernelregimes,gunasekar2018characterizing}. If properly performed, the (induced) dynamics in the predictor space can now be analyzed by a classical algorithm: mirror descent (or mirror flow). Implicit regularization is demonstrated by showing that the limit point satisfies a KKT (Karush–Kuhn–Tucker) condition with respect to minimizing some regularizer $R(\cdot)$ among all possible solutions. 

At first,
we were unable
to express the gradient dynamics in Eq.~\eqref{eq:grad-on-each-group} in terms of $\vec{w}(t)$ (i.e., in the predictor space), due to complicated interactions between $\vec{u}$ and $\vec{v}$. 
This hints that the training trajectory induced by an overparameterized DGLNN may not be analyzed by mirror flow techniques. In fact, we prove a stronger negative result, and rigorously show that the corresponding dynamics \emph{cannot} be recast as a mirror flow. Therefore, we conclude that our subsequent analysis techniques are necessary and do not follow as a corollary from existing approaches.

We first list two definitions from differential topology below.
\begin{definition}
Let $M$ be a smooth submanifold of $\mathbb{R}^D$. Given two $C^1$ vector fields of $X, Y$ on $M$, we define the \emph{Lie Bracket} of $X$ and $Y$ as $[X, Y](x) \coloneqq \partial Y(x) X(x) - \partial X(x) Y(x)$.
\end{definition}

\begin{definition}
Let $M$ be a smooth submanifold of $\mathbb{R}^D$. A $C^2$ parameterization $G:M\to\mathbb{R}^d$ is said to be commuting iff for any $i,j\in [d]$, the Lie Bracket $[\nabla G_i,\nabla G_j](x) = 0$ for all $x\in M$.
\end{definition}

The parameterization studied in most existing works on diagonal networks is separable, meaning that each parameter only affects one coordinate in the predictor space. In DGLNN, the parameterization is not separable, due to the shared parameter $\vec{u}$ within each group.
We formally show that it is indeed not commuting.

\begin{lemma}
$G(\cdot)$ is not a commuting parameterization.
\label{lemma:non-commuting}
\end{lemma}

Non-commutativity of the parameterization implies that moving along $-\nabla G_i$ and then $-\nabla G_j$ is different with moving with $-\nabla G_j$ first and then $-\nabla G_i$. This causes extra difficulty in analyzing the gradient dynamics.  \citet{li2022implicit} study the equivalence between gradient flow on reparameterized models and mirror flow, and show that a commuting parameterization is a sufficient condition for when a gradient flow with certain parameterization simulates a mirror flow. A complementary necessary condition is also established on the Lie algebra generated by the gradients of coordinate functions of $G$ with order higher than 2. We show that the parameterization $G(\cdot)$ violates this necessary condition.

\begin{theorem}
There exists an initialization $[\vec{u}^\top_{init},\vec{v}_{init}^\top] \in \mathbb{R}^L_+\times \mathbb{R}^p$ and a time-dependent loss $L_t$ such that gradient flow under $L_t\odot G$ starting from $[\vec{u}^\top_{init},\vec{v}_{init}^\top]$ cannot be written as a mirror flow with respect to any Legendre function $R$ under the loss $L_t$.
\label{thm:not-mf}
\end{theorem}

The detailed proof is deferred to the Appendix. Theorem~\ref{thm:not-mf} shows that the gradient dynamics implied in DGLNN cannot be emulated by mirror descent. Therefore, a different technique is needed to analyze the gradient dynamics and any associated implicit regularization effect.

\subsection{Layer Balancing and Gradient Flow}
Let us first introduce relevant quantities.  Following our reparameterization, we rewrite the true parameters for each group $l$ as 
\[\wstar_l = (u_l^\star)^2 \vec{v}^\star_l, 
\quad 
\norm{\vec{v}_l^\star}_2 = 1,
\quad
\vec{v}^\star_l\in\mathbb{R}^{p_l}.\]

The support is defined on the group level, where $S=\{l\in[L] : u_l^\star > 0\}$ and the support size is defined as $s=|S|$. We  denote 
$
u_{max}^\star = \max\{u_l^\star| l \in S\}$,
and 
$u_{min}^\star = \min\{u_l^\star| l \in S\}$.

The gradient dynamics in our reparameterization does not preserve $\norm{\vec{v}_l(t)}_2=1$, which causes difficulty to identify the magnitude of each $u_l$ and $\norm{\vec{v}_l(t)}_2$. \citet{du2018algorithmic} and \citet{arora2018optimization} show that the gradient flow of multi-layer homogeneous functions effectively enforces the differences between squared norms across different layers to remain invariant. Following the same idea, we discover a similar balancing effect in DGLNN between the parameter $\vec{u}$ and $\vec{v}$.

\begin{lemma}
For any $l\in[L]$, we have
\[
\frac{d}{d t}\left(\frac12 u_l^2 - \norm{\vec{v}_l}^2 \right) = 0.
\]
\label{lemma:balancing}
\end{lemma}

The balancing result eliminates the identifiability issue on the magnitudes. As the coordinates within one group affect each other, the direction which controls the growth rate of both $\vec{u}$ and $\vec{v}$ need to be determined as well.
\begin{lemma}
If the initialization 
$\vec{v}_l(0)$ is proportional to   $\frac1n \matrix{X}_l^\top \vec{y}$,
then
\[
\left\langle \frac{\vec{v}_l(0)}{\norm{\vec{v}_l(0)}}, \vec{v}_l^\star\right\rangle 
\geq
1 - \left(\delta_{in} 
+ L\delta_{out} 
+ \norm{\frac1n \matrix{X}_l^\top\vec{\xi}}_2 /(u_{l}^\star)^2\right)^2.
\]
\label{lemma:warm-up-init}
\end{lemma}
Note that this initialization can be obtained by a single step of gradient descent with $\vec{0}$ initialization. Lemma~\ref{lemma:warm-up-init} suggests the direction is close to the truth at the initialization. We can further normalize it to be $\norm{\vec{v}_l(0)}_2^2 = \frac12 u_l^2(0)$ based on the balancing criterion. The magnitude equality, $\norm{\vec{v}_l(t)}_2^2 = \frac12 u_l^2(t)$, is preserved by Lemma \ref{lemma:balancing}.
However, ensuring the closeness of the direction throughout the gradient flow presents significant technical difficulties. 
That said, we are able to present a meaningful implicit regularization result of the gradient flow under orthogonal (and noisy) settings.


\begin{theorem}
\label{thm:gf-es}
Fix $\epsilon > 0$. Consider the case where $\frac1n \matrix{X}_l^\top \matrix{X}_l = \matrix{I}$,
$\frac1n \matrix{X}_l^\top \matrix{X}_{l'} = \matrix{O}, l\neq l'$, the initialization $u_l(0) = \theta < \frac{\epsilon}{2(u_{max}^\star)^2}$ and $\vec{v}_l(0) = \eta_l \frac1n \matrix{X}_l^\top \vec{y}$ with $\norm{\vec{v}_l(0)}_2^2= \frac12 \theta^2, \forall l\in [L]$, there exists 
an lower bound and upper bound of the time $T_l < T_u$
in the gradient flow in Eq.~\eqref{eq:grad-on-each-group},
such that for any $T_l\leq t \leq T_u$ we have
\begin{align*}
    \norm{u_l^2(t)\vec{v}_l(t) - \vec{w}_l^\star}_\infty
    \lesssim
\begin{cases}
  \norm{\frac1n \X^\top \xxi}_\infty \vee  \epsilon, &\text{if }l\in S.\\
  \theta^{3/2}, &\text{if }l\notin S.
\end{cases}
\end{align*}
\end{theorem}
Theorem~\ref{thm:gf-es} states
the error bounds for the estimation of the \textit{true}
weights $\wstar$. 
For entries outside the (true) support, the error is controlled by $\theta^{3/2}$.
When $\theta$ is small, the algorithm keeps all non-supported entries to be close to zero through iterations
while maintaining the guarantee for supported entries.
Theorem~\ref{thm:gf-es} shows that under the assumption of orthogonal design, gradient flow with early stopping is able to obtain the solution with group sparsity.

\section{Gradient Descent with Weight Normalization}
\label{sec:gd-norm}

\begin{algorithm}[!ht]
\caption{Gradient descent with weight normalization}\label{alg:gd-norm}
\begin{algorithmic}
\State \textbf{Initialize:} $\vec{u}(0)=\alpha \vec{1}$, unit norm initialization  $\vec{v}_l(0)$ for each $l\in [L]$, 
$\eta_{l,t} = \frac{1}{u^4_l(t)}$.
\For{$t=0$ to $T$}
\State $\vec{z}(t+1) = \vec{v}(t) - \eta_{l,t} \nabla_{\vec{v}} \mathcal{L}(\vec{u}(t),\vec{v}(t))$
\State $\vec{v}_l(t+1) = \frac{\vec{z}_l(t+1)}{\norm{\vec{z}_l(t+1)}_2}, \forall l\in [L]$
\State $\vec{u}(t+1) = \vec{u}(t) - \gamma \nabla_{\vec{u}} \mathcal{L}(\vec{u}(t),\vec{v}(t+1))$
\If{the early stopping criterion is satisfied}
\State stop
\EndIf
\EndFor
\end{algorithmic}
\end{algorithm}

We now seek a more practical algorithm with more general assumptions and requirements on initialization. To speed up the presentation, we will directly discuss the corresponding variant of (the more practical) gradient descent instead of gradient flow.
When standard gradient descent is applied on DGLNN, initialization for directions is very crucial; The algorithm may fail even with a very small initialization when the direction is not accurate, as shown in Appendix~\ref{sec:more-numerical-res}. The balancing effect (Lemma~\ref{lemma:balancing}) is sensitive to the step size, and errors may accumulate \citep{du2018algorithmic}.

Weight normalization as a commonly used training technique has been shown to be helpful in stabilizing the training process.  The identifiability of the magnitude is naturally resolved by weight normalization on each $\vec{v}_l$. Moreover, weight normalization allows for a larger step size on $\vec{v}$, which makes the direction estimation at each step behave like that at the origin point. This removes the restrictive assumption of orthogonal design. With these intuitions in mind, we study the gradient descent algorithm with weight normalization on $\vec{v}$ summarized in Algorithm \ref{alg:gd-norm}.
One advantage of our algorithm is that it converges with \emph{any} unit norm initialization $\vec{v}_l(0)$.
The step size on $\vec{u}(t)$ is chosen to be small enough in order to enable the incremental learning, whereas the step size on $\vec{v}(t)$ is chosen as $\eta_{l,t} = \frac{1}{u^4_l(t)}$ as prescribed by our theoretical investigation.
For convenience, we define
\[
\zeta = 80\left(\norm{\frac1n \X^\top \xxi}_\infty \vee \epsilon\right),
\]
for a precision parameter $\epsilon > 0$. 
The convergence of Algorithm \ref{alg:gd-norm} is formalized as follows:

\begin{theorem}
Fix $\epsilon > 0$.
Consider Algorithm~\ref{alg:gd-norm} with 
\[
u_l(0) = \alpha < \frac{\epsilon^4\wedge 1}{(u_{max}^\star)^8}  \wedge \frac{1}{80L}(u_{min}^\star)^2 \wedge \frac{\epsilon}{L}, \quad \forall l\in[L],
\]
any unit-norm initialization on $\vec{v}_l$ for each $l\in[L]$ and $\gamma \leq \frac{1}{20(u_{max}^\star)^2}$.
Suppose Assumption~\ref{assumptions} is satisfied with
$\delta_{in} \leq  \frac{(u_{min}^\star)^2}{120(u_{max}^\star)^2}$ and $\delta_{out} \leq  \frac{(u_{min}^\star)^2}{120s(u_{max}^\star)^2}$. There exist a lower bound on the number of iterations 
\[T_{lb} =
 \frac{\log \frac{(u^\star_{max})^2}{2\alpha^2}}
{ 2\log (1+ \frac{\gamma}{2}(\zeta \vee (u^\star_{min})^2))} + \left\lfloor \log_2 \frac{(u_{max}^\star)^2}\zeta\right\rfloor \frac{5}{2\gamma (\zeta \vee (u_{min}^\star)^2)},
\]
and an upper bound
\[
T_{ub} \geq \frac5{16\gamma (\zeta \vee (u_{min}^\star)^2)}\log \frac1{\alpha^4},
\]
such that $T_{lb}\le T_{ub}$ and for any $T_{lb}\leq t\leq T_{ub}$,
\[
\norm{u_l^2(t)\vec{v}_l(t) - \vec{w}_l^\star}_\infty
\lesssim
\begin{cases}
  \norm{\frac1n \X^\top \xxi}_\infty \vee  \epsilon,
  &\text{if }l\in S\\
  \alpha,
  &\text{if }l\notin S
\end{cases}.
\]
\label{thm:convergence-alg-1}
\end{theorem}

Similarly as Theorem \ref{thm:gf-es}, Theorem~\ref{thm:convergence-alg-1} states
the error bounds for the estimation of the \textit{true}
weights $\wstar$. 
When $\alpha$ is small, the algorithm keeps all non-supported entries to be close to zero through iterations
while maintaining the guarantee for supported entries.
Compared to the works on implicit (unstructured) sparse regularization \citep{vaskevicius2019implicit, chou2021more}, 
our assumption on the incoherence parameter $\delta_{out}$ scales with $1/s$, where $s$ is the number of non-zero groups, instead of the total number of non-zero entries. Therefore, the relaxed bound on $\delta_{out}$ implies an improved sample complexity, 
which is also observed experimentally in Figure~\ref{fig:group_vs_sparse}. 
We now state a corollary in a common setting with independent random noise, where (asymptotic) recovery of $\wstar$ is possible.


\begin{definition}
A random variable $Y$ is $\sigma$-sub-Gaussian if for all $t \in \mathbb{R}$ there exists $\sigma > 0$ such that
\[
\mathbb{E} e^{tY} \leq e^{\sigma^2 t^2/2}.
\]
\end{definition}

\begin{corollary}
\label{cor:subgaussian-noise}
Suppose the noise vector $\vec{\xi}$ has independent $\sigma^2$-sub-Gaussian entries and $\epsilon=2\sqrt{\frac{\sigma^2\log(2p)}{n}}$. Under the assumptions of Theorem \ref{thm:convergence-alg-1}, Algorithm~\ref{alg:gd-norm} produces $\vec{w}(t) =(\D\vec{u}(t))^{\circ2} \odot \vec{v}(t)$ that satisfies $\norm{\vec{w}(t) - \wstar}_2^2\lesssim (s\sigma^2 \log p)/n$ with probability at least $1-1/(8p^3)$ for any $t$ such that $T_{lb} \leq t \leq T_{ub}$.
\end{corollary}

Note that the error bound we obtain is minimax-optimal.
Despite these appealing properties of Algorithm \ref{alg:gd-norm}, our theoretical results require
a large step size on each $\vec{v}_l(t)$, which may cause instability at later stages of learning.
We observe this instability numrerically (see  Figure~\ref{fig:instability_alg1}, Appendix~\ref{sec:more-numerical-res}). Although the estimation error of $\vec{w}^\star$ remains small (which aligns with our theoretical result),
individual entries in $\vec{v}$ may fluctuate considerably.
Indeed, the large step size is mainly introduced to maintain a strong directional information extracted from the gradient of $\vec{v}_l(t)$ so as to stabilize the updates of $\vec{u}(t)$ at the early iterations. Therefore, we also propose
Algorithm~\ref{alg:gd-norm-decrease}, a variant of Algorithm~\ref{alg:gd-norm}, where we decrease the step size after a certain number of iterations.

\begin{algo}{2}
Run Algorithm~\ref{alg:gd-norm} with the same setup till each $u_l(t),l\in[L]$ gets roughly accurate, set 
$\eta_{l,t} = \eta.$
Continue Algorithm~\ref{alg:gd-norm} until early stopping criterion is satisfied.
\label{alg:gd-norm-decrease}
\end{algo}

\begin{theorem}
\label{thm:convergence-alg-2}
Under the assumptions of Theorem~\ref{thm:convergence-alg-1} with replacing the condition on $\delta$'s by $\delta_{in} \leq \frac{\sqrt\zeta (u_{min}^\star)^2}{120 (u_{max}^\star)^3}$ and $\delta_{out} \leq \frac{\sqrt\zeta (u_{min}^\star)^2}{120 s (u_{max}^\star)^3}$, we apply Algorithm~\ref{alg:gd-norm-decrease} with $\eta_{l,t} = \frac{1}{u^4(t)}$ at the beginning, and $\eta_{l,t} = \eta \leq \frac{4}{9(u_{max}^\star)^2}$ after $\forall l\in[L], u^2_l(t) \geq \frac12 (u^\star_l)^2$, then with the same $T_{lb}$ and $T_{ub}$,
we have that for any $T_{lb}\leq t \leq T_{ub}$,
\[
\norm{u_l^2(t)\vec{v}_l(t) - \vec{w}_l^\star}_\infty
\lesssim
\begin{cases}
  \norm{\frac1n \X^\top \xxi}_\infty \vee  \epsilon, &\text{if }l\in S.\\
  \alpha, &\text{if }l\notin S.
\end{cases}
\]
\end{theorem}


In Theorem \ref{thm:convergence-alg-2}, the criterion to decrease the step size is: $u^2_l(t) \geq \frac12 (u^\star_l)^2, \forall l\in[L]$. Once this criterion is satisfied, our proof indeed ensures that it would hold for at least up to the early stopping time $T_{ub}$ specified in the theorem.
In practice, since $u_l^\star$'s are unknown, we can switch to a more practical criterion:
$\max\limits_{l\in[L]}\{{|u_l(t+1) - u_l(t)|}/{|u_l(t) + \varepsilon|}\} < \tau$ for some pre-specified tolerance $\tau>0$ and small value $\varepsilon>0$ as the criterion for changing the step size.
The motivation of this criterion is further discussed in Appendix~\ref{sec:proof-for-theorems}. 
The error bound remains the same as Theorem~\ref{thm:convergence-alg-1}.
The change in step size requires a new way to study the gradient dynamics of directions with perturbations. With our proof technique,
Theorem \ref{thm:convergence-alg-2} requires a smaller bound on $\delta$'s (see Lemma~\ref{lemma:direction-bound} versus Lemma~\ref{lemma:direction-bound-optimal} in Appendix~\ref{sec:general-case} for details). We believe it is a proof artifact and leave the improvement for future work.

\textbf{Connection to standard sparsity.}  Consider the degenerate case where each group size is 1. Our reparameterization, together with the normalization step, can roughly be interpreted as 
$w_i \approx u_i^2 \mathop{\rm sgn}(v_i),$
which is different from the power-reparameterization $w_i = u_i^N - v_i^N, N\geq2$ in \citet{vaskevicius2019implicit} and \citet{li2021implicit}. This also shows why a large step size on $v_i$ is needed at the beginning. If the initialization on $v_i$ is incorrect, the sign of $v_i$ may not move with a small step size. 


\section{Simulation Studies}
\label{sec:simulation}
We conduct various experiments on simulated data to support our theory. Following the model in Section \ref{sec:setup}, we sample the entries of $\X$ i.i.d.~using Rademacher random variables and the entries of the noise vector $\xxi$ i.i.d. under $N(0,\sigma^2)$.
We set $\sigma=0.5$ throughout the experiments. 

\begin{figure}[ht!]
    \centering
    \includegraphics[width=.9\linewidth]{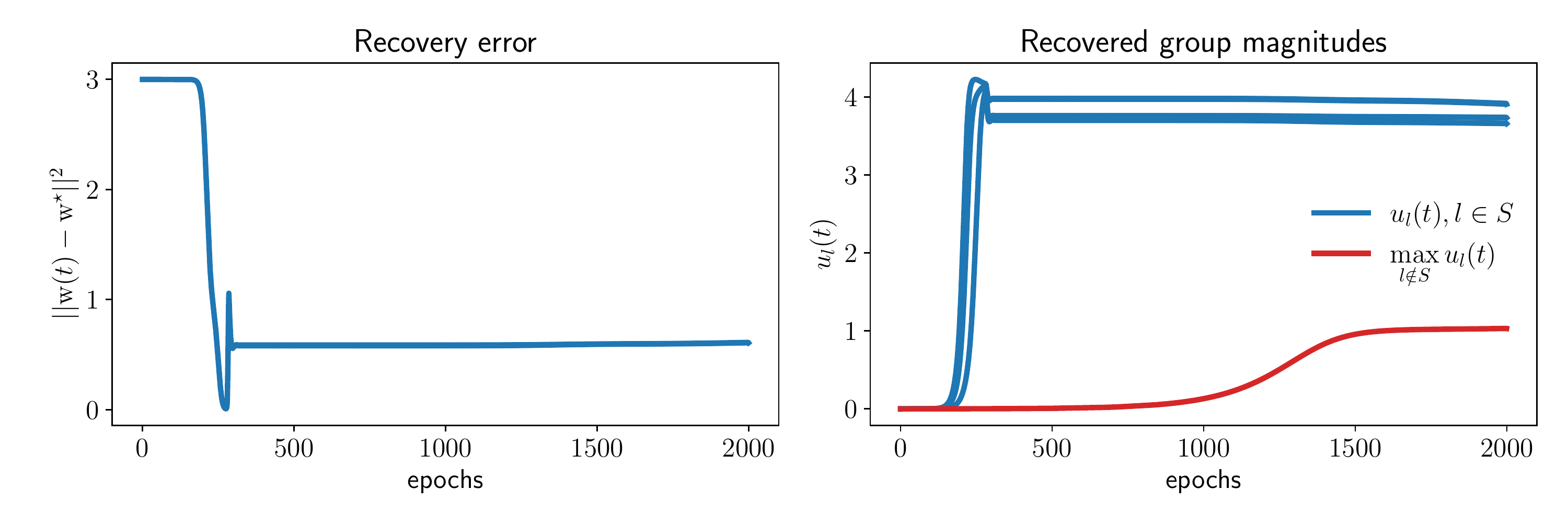}
    \caption{Convergence of Algorithm~\ref{alg:gd-norm}. The entries on the support are all 10.}
    \label{fig:convergence_alg1}
\end{figure}

\begin{figure}[ht!]
    \centering
    \includegraphics[width=.9\linewidth]{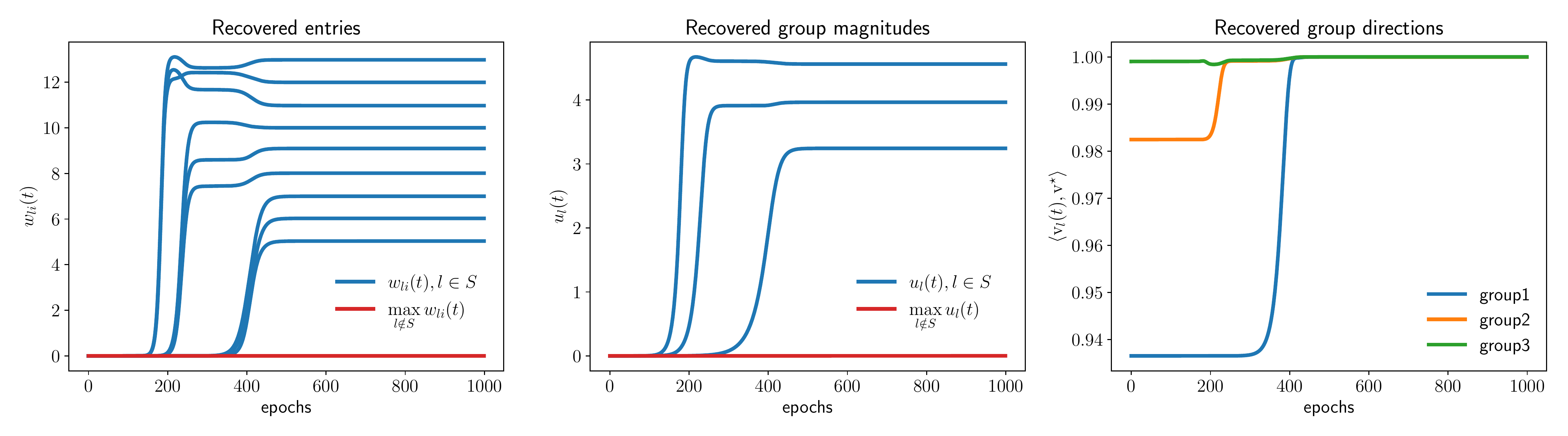}
    \caption{Convergence of Algorithm~\ref{alg:gd-norm-decrease}. The entries on the support are from 5 to 13.}
    \label{fig:convergence_alg2}
\end{figure}

\textbf{The effectiveness of our algorithms.} We start by demonstrating the convergence of the two proposed algorithms. In this experiment, we set $n=150$ and $p=300$. The number of non-zero entries is $9$, divided into $3$ groups of size $3$. We run both Algorithms~\ref{alg:gd-norm} and \ref{alg:gd-norm-decrease} 
with the same initialization $\alpha = 10^{-6}$. The step size $\gamma$ on $\vec{u}$ and decreased step size $\eta$ on $\vec{v}$ are both $10^{-3}$. In Figure~\ref{fig:convergence_alg1}, we present the recovery error of $\wstar$ on the left, and recovered group magnitudes on the right. As we can see, early stopping is crucial for reaching the structured sparse solution. In Figure~\ref{fig:convergence_alg2}, we present the recovered entries, recovered group magnitudes and recovered directions for each group from left to right. In addition to convergence, we also observe an incremental learning effect. 

\begin{figure}[ht!]
    \centering
    \includegraphics[width=.9\linewidth]{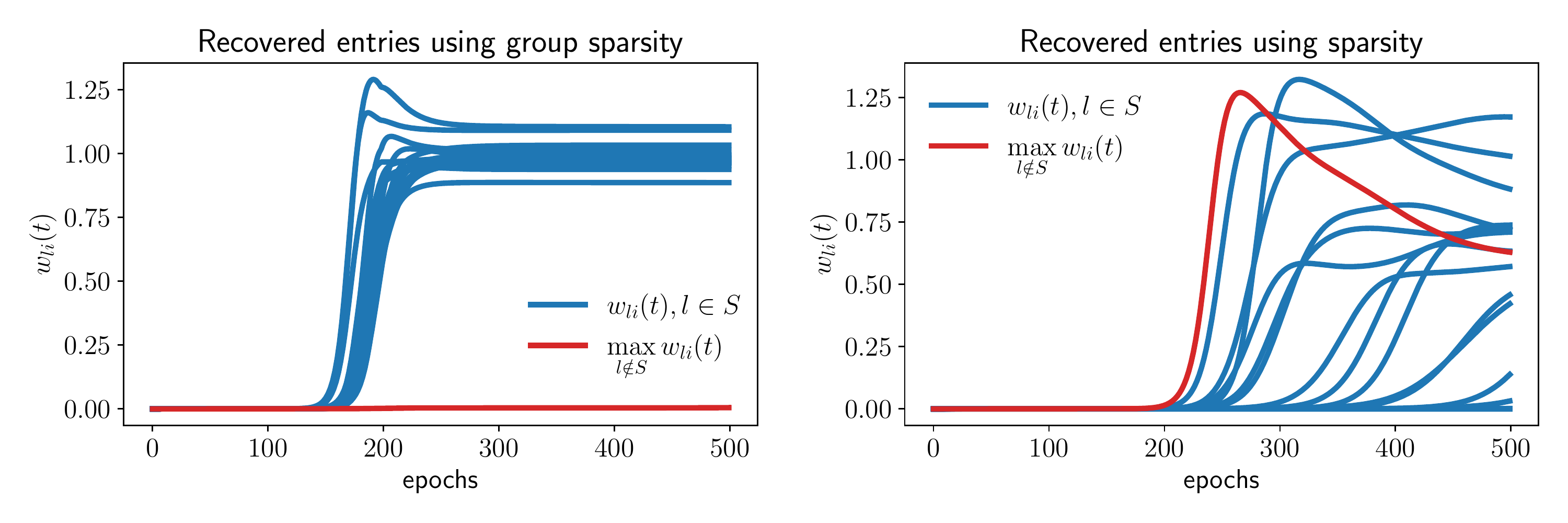}
    \caption{Comparison with reparameterization using standard sparsity. $n=100, p=500$.}
    \label{fig:group_vs_sparse}
\end{figure}

\textbf{Structured sparsity versus standard sparsity.} From our theory, we see that the block incoherence parameter scales with the number of non-zero groups, as opposed to the number of non-zero entries. As such, we can expect an improved sample complexity over the estimators based on unstructured sparse regularization. We choose a larger support size of 16. The entries on the support are all 1 for simplicity. We apply our Algorithm~\ref{alg:gd-norm-decrease} with group size $4$. The result is shown in Figure~\ref{fig:group_vs_sparse} (left). We compare with the method in \citet{vaskevicius2019implicit} with parameterization $\vec{w} = \vec{u}^{\circ 2} - \vec{v}^{\circ 2}$, designed for unstructured sparsity. We display the result in the right figure, where interestingly, that algorithm fails to converge because of an insufficient number of samples.

\begin{figure}[ht!]
    \centering
    \includegraphics[width=.9\linewidth]{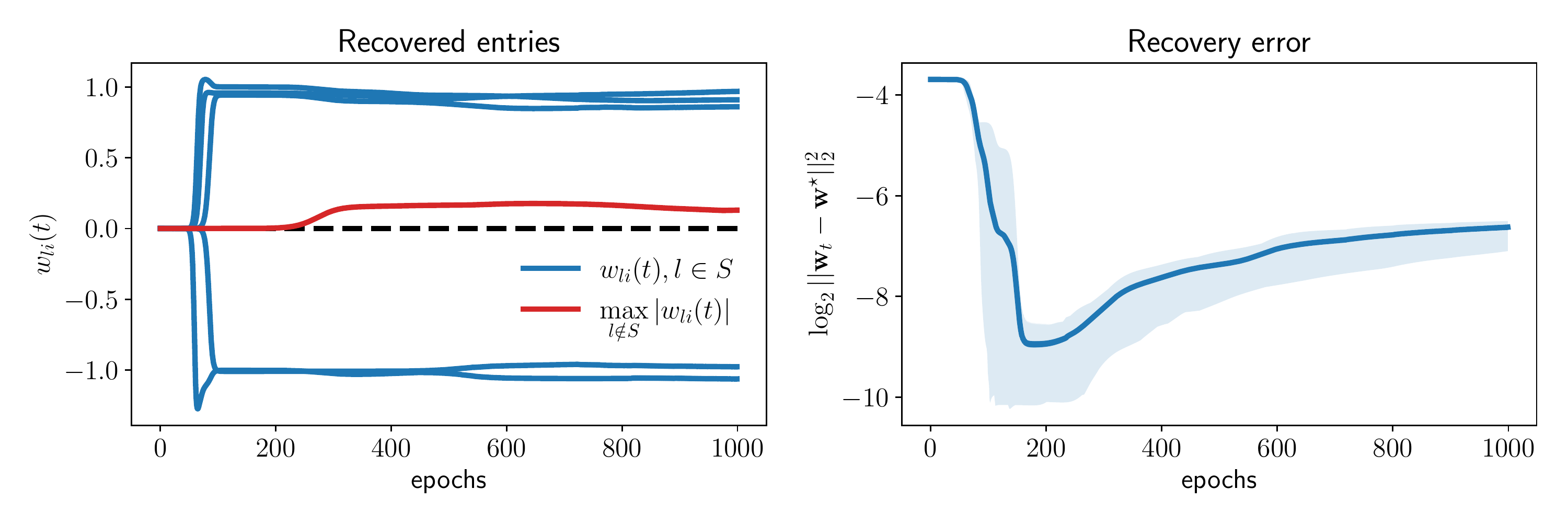}
    \caption{Degenerate case when each group size is 1. The $\log \ell_2$-error plot is repeated 30 times, and the mean is depicted. The shaded area indicates the region between the $25^\textrm{th}$ and $75^\textrm{th}$ percentiles.}
    \label{fig:degenerated_case}
\end{figure}

\textbf{Degenerate case.} In the degenerate case where each group is of size 1, our reparameterization takes a simpler form $w_i \approx u_i^2 {\rm sgn}(v)$, i.e., 
due to weight normalization, our method normalizes $v$ to $1$ or $-1$ after each step. We demonstrate the efficacy of our algorithms even in the degenerate case. We set $n=80$ and $p=200$. The entries on the support are $[1,-1,1,-1,1]$ with both positive and negative entries. We present the coordinate plot and the recovery error in Figure~\ref{fig:degenerated_case}.

\section{Discussion}
In this paper, we show that implicit regularization for group-structured sparsity can be obtained by gradient descent (with weight normalization) for a certain, specially designed network architecture.
Overall, we hope that such analysis further enhances our understanding of neural network training. Future work includes relaxing the assumptions on $\delta$'s in Theorem~\ref{alg:gd-norm-decrease}, 
and rigorous analysis of modern grouping architectures as well as power parametrizations.


\subsubsection*{Acknowledgments}
This work was supported in part by the National Science Foundation under grants CCF-1934904, CCF-1815101, and CCF-2005804.

\bibliography{iclr2023_conference}
\bibliographystyle{iclr2023_conference}

\newpage
\appendix

\section{Geometric properties of the parametrization}
We start by calculating the vector field induced by the parameterization $G(\cdot)$.
\[
\nabla G_i([\vec{u}^\top,\vec{v}^\top])
= 2 u_{g(i)} v_i \vec{e}_{g(i)} + u_{g(i)}^2 \vec{e}_{L+i},
\]
where $\vec{e}_i \in \mathbb{R}^{L+p}$ is only 1 on $i^\textrm{th}$ entry and 0 elsewhere, and
\[
\nabla^2 G_i([\vec{u}^\top,\vec{v}^\top])
= 2v_i \matrix{E}_{g(i),g(i)} + 2u_{g(i)} \matrix{E}_{g(i),L+i} + 2u_{g(i)} \matrix{E}_{L+i, g(i)}, 
\]
where $\matrix{E}_{i,j}\in \mathbb{R}^{(L+p)\times(L+p)}$
is the one-hot matrix for $i^\textrm{th}$ row and $j^\textrm{th}$ column.
For $i\neq j$ s.t. $g(i)=g(j)$,
\begin{align*}
\nabla^2 G_i([\vec{u}^\top,\vec{v}^\top]) \nabla G_j([\vec{u}^\top,\vec{v}^\top])
&=
(2v_i \matrix{E}_{g(i),g(i)} + 2u_{g(i)} \matrix{E}_{g(i),L+i} + 2u_{g(i)} \matrix{E}_{L+i, g(i)})\\
&\cdot(2 u_{g(j)} v_j \vec{e}_{g(j)} + u_{g(j)}^2 \vec{e}_{L+j})\\
&=4u_{g(j)} v_i v_j \matrix{e}_{g(i)}
+4u_{g(i)} u_{g(j)} v_j \matrix{e}_{L+i}\\
&= 4u_{g(i)} v_i v_j \matrix{e}_{g(i)}
+4u_{g(i)}^2 v_j \matrix{e}_{L+i},
\end{align*}
similarly,
\[
\nabla^2 G_j([\vec{u}^\top,\vec{v}^\top]) \nabla G_i([\vec{u}^\top,\vec{v}^\top])
= 
4u_{g(i)} v_i v_j \matrix{e}_{g(i)}
+4u_{g(i)}^2 v_i \matrix{e}_{L+j}.
\]

\textbf{Proof for Lemma \ref{lemma:non-commuting}.}
For two indices within the same group, i.e, $i\neq j$ and $g(i)=g(j)$, we obtain that
\begin{align*}
    [\nabla G_i, \nabla G_j]([\vec{u}^\top,\vec{v}^\top])
    =&
    \nabla^2 G_j([\vec{u}^\top,\vec{v}^\top]) \nabla G_i([\vec{u}^\top,\vec{v}^\top])
    -
    \nabla^2 G_i([\vec{u}^\top,\vec{v}^\top]) \nabla G_j([\vec{u}^\top,\vec{v}^\top])
    \\
    &= 4u_{g(i)}^2 v_j \matrix{e}_{L+i}
    - 4u_{g(i)}^2 v_i \matrix{e}_{L+j},
\end{align*}
which is not always $\vec{0}$ when $v_i \neq v_j$. Therefore, $G(\cdot)$ is not commuting.
\qed

\textbf{Proof for Theorem \ref{thm:not-mf}.}
For $i\neq j$ and $g(i)\neq g(j)$, we have
\[
[\nabla G_i, \nabla G_j]([\vec{u}^\top,\vec{v}^\top])
= \vec{0}.
\]
For $i\neq j$ and $g(i) = g(j)$, we have that
\[
[\nabla G_i, \nabla G_j]([\vec{u}^\top,\vec{v}^\top])
= v_j \nabla G_i - v_i \nabla G_j \in {\rm span}\{\nabla G_i\}_{i=1}^p.
\]
By Corollary 4.13 in \citep{li2022implicit} and Lemma~\ref{lemma:non-commuting}, we show that there exists and initialization and a time-dependent loss that the gradient flow can not be analyzed by mirror flow.
\qed

Alternatively,
we can show directly that the necessary condition in Theorem 4.10 in \cite{li2022implicit} is violated, i.e., 
\[
\langle \nabla G_j, [\nabla G_i, [\nabla G_i, \nabla G_j]]\rangle([\vec{u}^\top,\vec{v}^\top]) \neq 0
\]
for some $[\vec{u}^\top,\vec{v}^\top]$ in every open set $M$.

We first obtain that
\begin{align*}
    \nabla [\nabla G_i, \nabla G_j]([\vec{u}^\top,\vec{v}^\top])
    &= 8 u_{g(i)} v_j \matrix{E}_{L+i, g(i)} + 4 u_{g(i)}^2 \matrix{E}_{L+i,L+j}\\
    &-
    8 u_{g(i)} v_i \matrix{E}_{L+j, g(i)} - 4 u_{g(i)}^2 \matrix{E}_{L+j,L+i}.
\end{align*}
Therefore,
\begin{align*}
    [\nabla G_i, [\nabla G_i, \nabla G_j]]([\vec{u}^\top,\vec{v}^\top])
    &= \nabla [\nabla G_i, \nabla G_j]([\vec{u}^\top,\vec{v}^\top]) \nabla G_i([\vec{u}^\top,\vec{v}^\top])\\
    &- \nabla^2 G_i([\vec{u}^\top,\vec{v}^\top]) [\nabla G_i, \nabla G_j]([\vec{u}^\top,\vec{v}^\top])\\
    &= (8 u_{g(i)} v_j \matrix{E}_{L+i, g(i)} + 4 u_{g(i)}^2 \matrix{E}_{L+i,L+j}\\
    &-
    8 u_{g(i)} v_i \matrix{E}_{L+j, g(i)} - 4 u_{g(i)}^2 \matrix{E}_{L+j,L+i})\\
    &\cdot
    (2 u_{g(i)} v_i \vec{e}_{g(i)} + u_{g(i)}^2 \vec{e}_{L+i})\\
    &- (2v_i \matrix{E}_{g(i),g(i)} + 2u_{g(i)} \matrix{E}_{g(i),L+i} + 2u_{g(i)} \matrix{E}_{L+i, g(i)})\\
    &\cdot (4u_{g(i)}^2 v_j \matrix{e}_{L+i}
    - 4u_{g(i)}^2 v_i \matrix{e}_{L+j})\\
    &= 16 u_{g(i)}^2 v_i v_j \vec{e}_{L+i}  - 16 u_{g(i)}^2 v_i^2 \vec{e}_{L+j} - 4 u_{g(i)}^4 \vec{e}_{L+j}
    -8 u_{g(i)}^3 v_j \vec{e}_{g(i)}\\
    &=16 u_{g(i)}^2 v_i v_j \vec{e}_{L+i}  - 
    (16 u_{g(i)}^2 v_i^2 +4 u_{g(i)}^4)\vec{e}_{L+j} 
    -8 u_{g(i)}^3 v_j \vec{e}_{g(i)}.
\end{align*}
Hence,
\begin{align*}
    &\langle \nabla G_j, [\nabla G_i, [\nabla G_i, \nabla G_j]]\rangle([\vec{u}^\top,\vec{v}^\top])\\
    =& \langle 2u_{g(i)} v_j \vec{e}_{g(i)} + u_{g(i)}^2 \vec{e}_{L+j},  16 u_{g(i)}^2 v_i v_j \vec{e}_{L+i}  - 
    (16 u_{g(i)}^2 v_i^2 +4 u_{g(i)}^4)\vec{e}_{L+j} 
    -8 u_{g(i)}^3 v_j \vec{e}_{g(i)}\rangle\\
    =&
    -16 u_{g(i)}^4 v_j^2  - 16 u_{g(i)}^4 v_i^2 - 4 u_{g(i)}^6 < 0.
\end{align*}
By Theorem 4.10 in \cite{li2022implicit}, there exists an initialization such that no Legendre function $R$ is able to make the gradient flow be written as a mirror flow with respect to $R$.

\section{Proof for Analysis of Gradient Flow}
\label{sec:proof-gf}
\textbf{Proof for Lemma \ref{lemma:balancing}.}
Recall
\begin{align*}
    \frac{\partial \mathcal{L}}{\partial u_l}
    = -\frac2n u_l \vec{v}_l^\top \X_l^\top \vec{r}(t),
    \quad
    \frac{\partial \mathcal{L}}{\partial \vec{v}_l}
    = -\frac1n u_l^2 \X_l^\top\vec{r}(t).
\end{align*}

Therefore, we obtain that
\begin{align*}
    \frac{\partial \norm{\vec{v}_l(t)}^2}{\partial t}
    &= 2\vec{v}_l^\top(t) \frac{\partial \vec{v}_l(t)}{\partial t}
    =2\vec{v}_l^\top(t)
    \left(-\frac{\partial \mathcal{L}}{\partial \vec{v}_l}\right)\\
    &= \frac2n u_l^2 \vec{v}_l^\top(t) \X_l^\top \vec{r}(t)\\
    &= u_l \left(-\frac{\partial \mathcal{L}}{\partial u_l}\right)
    = \frac{\partial \frac12 u_l^2(t)}{\partial t}.
\end{align*}
\qed

\textbf{Proof for Lemma~\ref{lemma:warm-up-init}.}
We start with decomposing $\vec{v}_l(0)$
\begin{align*}
    \vec{v}_l(0) = \eta\frac1n \matrix{X}_l^\top \vec{y}
    &=  \eta \vec{w}_l^\star + \eta\left(\frac1n\X_l^\top \X -\matrix{I}\right) \vec{w}_l^\star
    + \eta\sum_{l'\neq l} \frac1n\matrix{X}_l^\top\matrix{X}_{l'} \vec{w}_{l'}^\star
    + \eta \frac1n\matrix{X}_l^\top \vec{\xi}\\
    & = \eta \vec{w}_l^\star + \eta \vec{b}_l.
\end{align*}
With this decomposition, we have that
\begin{align*}
\langle \vec{v}_l(0), \vec{v}_l^\star \rangle^2 
&= \eta^2 ( (u_l^\star)^2 + \langle \vec{b}_l,\vec{v}_l^\star\rangle)^2\\
\norm{\vec{v}_l(0)}_2^2
&= \eta^2 ( (u_l^\star)^4 + 2\langle \vec{b}_l,\vec{w}_l^\star\rangle + \norm{\vec{b}_l}_2^2).
\end{align*}
Therefore,
\begin{align*}
    \frac{\langle \vec{v}_l(0), \vec{v}_l^\star \rangle^2 }{\norm{\vec{v}_l(0)}_2^2}
    &=
    \frac{\eta^2 ( (u_l^\star)^2 + \langle \vec{b}_l,\vec{v}_l^\star\rangle)^2}{\eta^2 ( (u_l^\star)^4 + 2\langle \vec{b}_l,\vec{w}_l^\star\rangle + \norm{\vec{b}_l}_2^2)}\\
    &=1 - \frac{\norm{\vec{b}_l}_2^2 - \langle \vec{b}_l,\vec{v}_l^\star\rangle^2}{(u_l^\star)^4 + 2\langle \vec{b}_l,\vec{w}_l^\star\rangle + \norm{\vec{b}_l}_2^2}\\
    &= 
    1 - \frac{\norm{\vec{b}_l / (u_l^\star)^2}_2^2 - \langle \vec{b}_l/ (u_l^\star)^2,\vec{v}_l^\star\rangle^2}{1 + 2\langle \vec{b}_l/ (u_l^\star)^2,\vec{v}_l^\star\rangle + \norm{\vec{b}_l/ (u_l^\star)^2}_2^2}\\
    &= 1 - \frac{1 - \langle \vec{b}_l/\norm{\vec{b}_l}, \vec{v}_l^\star \rangle^2 }{1 + 2 \norm{\vec{b}_l}/(u^\star_l)^2 \langle \vec{b}_l/\norm{\vec{b}_l}, \vec{v}_l^\star \rangle + \norm{\vec{b}_l}^2/(u_l^\star)^4} \norm{\vec{b}_l/(u_l^\star)^2}^2\\
    &\geq 1-\norm{\vec{b}_l/(u_l^\star)^2}_2^2,
\end{align*}
where last inequality is from
\begin{align*}
    \frac{1-\alpha^2}{\beta^2 + 2\alpha \beta +1}
    &= \frac{1}{\frac{\beta^2 + 2\alpha \beta +1}{1-\alpha^2}}
    =\frac{1}{1 + \frac{\beta^2 + 2\alpha \beta +\alpha^2}{1-\alpha^2}}\\
     &=\frac{1}{1 + \frac{(\alpha+\beta)^2}{1-\alpha^2}}
     \leq 1,
\end{align*}
for $0\leq \alpha \leq 1$.

Since 
\[
\norm{\vec{b}_l}_2 
\leq \delta_{in} (u_l^\star)^2 + L \delta_{out}(u_l^\star)^2 + \norm{\frac1n\X_l^\top \vec{\xi}}_2,
\]
we obtain that
\[
\left\langle \frac{\vec{v}_l(0)}{\norm{\vec{v}_l(0)}}, \vec{v}_l^\star\right\rangle 
\geq
1 - \left(\delta_{in} 
+ L\delta_{out} 
+ \norm{\frac1n \matrix{X}_l^\top\vec{\xi}}_2 /(u_{l}^\star)^2\right)^2.
\]
\qed

\begin{lemma}
\label{lemma:dir-not-change}
Consider a simplified case where $\frac1n \matrix{X}_l^\top \matrix{X}_l = \matrix{I}$,
$\frac1n \matrix{X}_l^\top \matrix{X}_{l'} = \matrix{O}, l\neq l'$, if $\vec{v}_l(0) = \eta \frac1n \matrix{X}_l^\top \vec{y}$, then 
\[
\vec{v}_l(t) = c\frac1n \matrix{X}_l^\top \vec{y},
\]
for some constant $c$.
\end{lemma}
\begin{proof}

From the gradient on the directions, we have that
\begin{align*}
    \frac{\partial \vec{v}_l(t)}{\partial t}
    &= \frac1n u_l^2(t) \X_l^\top\vec{r}(t)
    = \frac1n u_l^2(t) \X_l^\top \vec{y} - \frac1nu_l^2(t)\matrix{X}_l^\top \sum_{l'}\matrix{X}_{l'} u_{l'}^2(t) \vec{v}_{l'}(t)\\
    &= \frac1n u_l^2(t) \matrix{X}_l^\top \vec{y} - u_l^4(t) \vec{v}_l(t).
\end{align*}
Since $\vec{v}_l(0)$ is with the same direction as $\frac1n \X_l^\top \vec{y}$ at the initialization. Therefore, $\frac{\partial \vec{v}_l(t)}{\partial t}$ has the same direction as $\vec{v}_l(t)$. We obtain that $\vec{v}_l(t) = c\frac1n \matrix{X}_l^\top \vec{y}$ for some constant $c$.
\end{proof}

\begin{lemma}
\label{lemma:gf-error}
If the gradient flow satisfies
\[
\frac12 \frac{\partial u^2(t)}{\partial t}
    \leq u^6(t) + \sqrt{2}u^4(t) B
\]
for some constant $B > 0$, then for any $t\leq T = \frac{\log\frac1\theta}{2\theta^2 + \theta\sqrt2B}$ we have $u(t) \leq \sqrt\theta$ with initialization $u(0) = \theta$.
\end{lemma}
\begin{proof}
We wanted to find some time $T$ such that when $t\leq T$, $u(t)\leq \sqrt{\theta}$. Since the gradient is bounded from above, we obtain that
\begin{align*}
    \frac12 u^2(T) 
    &\leq \frac12 \theta^2 \cdot \exp\left(\int_0^T 2 u^4(t) + \sqrt{2} u^2(t)B dt\right)\\
    &\leq \frac12 \theta^2 \cdot \exp\left((2\theta^2 + \sqrt{2}\theta B)T\right)
    \leq \frac12 \theta.
\end{align*}
This gives us
\[
T \leq \frac{\log\frac1\theta}{2\theta^2 + \theta\sqrt2B}.
\]
\end{proof}

\begin{lemma}
\label{lemma:gf-convergence}
Fix any $\tau < \frac12$. Consider the gradient flow
\[
\frac12 \frac{\partial u^2(t)}{\partial t} \geq (1-2B)\sqrt{2} u^3(t) (u^\star)^2 - u^6(t) - \sqrt{2} u^3(t) B (u^\star)^2 
\]
for some constant $0<B<\frac1{10}$ with initialization $u(0) = \theta < \frac12 u^\star$, we have that 
\[
\left|\frac1{\sqrt{2}}u^3(t) - (u^\star)^2\right| < (1-3B-\tau) (u^\star)^2,
\]
after
\[
t \geq T =\frac{2^{1/3} (u^\star)^{4/3}}{\theta^2} \frac{1}{(1-6B)\sqrt{2} (u^\star)^2 \theta} +  \frac{2\log_2\frac{1}{2\tau}}{3(u^\star)^2 (1/2 -3B)\left(\sqrt{2}(1/2-3B)(u^\star)^2 \right)^{1/3}}.
\]
\end{lemma}
\begin{proof}
For any $T\geq 0$, we have that
\begin{align*}
    \frac12 u^2(T)
    \geq \frac12\theta^2 \cdot \exp\left( \int_0^T (1-2B)2\sqrt{2} u(t) (u^\star)^2 - 2 u^4(t) - 2\sqrt{2} u(t) B (u^\star)^2dt \right).
\end{align*}
When $u(t) < \frac12 u^\star$, we first aim to get $T_1$ such that $\frac1{\sqrt2} u^3(T_1) \geq \frac12 (u^\star)^2$.
Therefore, 
\begin{align*}
    &\frac12\theta^2 \cdot \exp\left( \int_0^T (1-2B)2\sqrt{2} u(t) (u^\star)^2 - 2 u^4(t) - 2\sqrt{2} u(t) B (u^\star)^2dt \right)\\
    &\geq
    \frac12\theta^2 \cdot \exp\left(\left((1-2B)2\sqrt{2} (u^\star)^2 - \sqrt{2}(u^\star)^2 - 2\sqrt{2}  B (u^\star)^2\right)\theta T_1\right)\\
    &\geq \frac12 \left(\frac{\sqrt{2}}{2}(u^\star)^2\right)^{2/3}.
\end{align*}
We obtain that 
\[
T \geq \frac{2^{1/3} (u^\star)^{4/3}}{\theta^2} \frac{1}{(1-6B)\sqrt{2} (u^\star)^2 \theta}.
\]
When $t \geq T_1$, we have that $\frac1{\sqrt2} u^3(t) \geq \frac12 (u^\star)^2$. Let us denote $\frac1{\sqrt2} u^3(0) = ((1-3B) - \eta)(u^\star)^2$, we wonder how many iterations $T_d$ are needed to make $\frac1{\sqrt2} u^3(T_d) \geq \left((1-3B) - \frac12\eta\right)(u^\star)^2$.
\begin{align*}
    &\frac12  \left(\sqrt{2}\left((1-3B) - \eta\right)(u^\star)^2 \right)^{2/3}
    \cdot \exp\left( \int_0^T (1-2B)2\sqrt{2} u(t) (u^\star)^2 - 2 u^4(t) - 2\sqrt{2} u(t) B (u^\star)^2dt \right)\\
    &\geq
    \frac12  \left(\sqrt{2}\left((1-3B) - \eta\right)(u^\star)^2 \right)^{2/3}
    \cdot 
    \exp\left(\left(\frac12\eta(u^\star)^2\right) \left(\sqrt{2}\left((1-3B) - \eta\right)(u^\star)^2 \right)^{1/3} T_2\right)\\
    &\geq \frac12  \left(\sqrt{2}\left((1-3B) - \eta\right)(u^\star)^2 \right)^{2/3}
    \cdot 
    \left(1+\left(\frac12\eta(u^\star)^2\right) \left(\sqrt{2}\left((1-3B) - \eta\right)(u^\star)^2 \right)^{1/3} T_2\right)\\
    &\geq 
    \frac12  \left(\sqrt{2}\left((1-3B) - \frac12 \eta\right)(u^\star)^2 \right)^{2/3}.
\end{align*}
Therefore,
\begin{align*}
    T_2 
    &\geq \frac{\left((1-3B) - \frac12 \eta\right)^{2/3} - \left((1-3B) - \eta\right)^{2/3}}{\left((1-3B) - \eta\right)^{2/3}} \frac{1}{\frac12\eta(u^\star)^2 \left(\sqrt{2}\left((1-3B) - \eta\right)(u^\star)^2 \right)^{1/3}}\\
    &\geq 
    \frac23 \frac{\frac12 \eta}{\frac12\eta(u^\star)^2 \left((1-3B) - \eta\right) \left(\sqrt{2}\left((1-3B) - \eta\right)(u^\star)^2 \right)^{1/3}}\\
    &\geq \frac{2}{3(u^\star)^2 (1/2 -3B)\left(\sqrt{2}(1/2-3B)(u^\star)^2 \right)^{1/3}}.
\end{align*}
Overall, we obtain that 
\[
\left|\frac1{\sqrt{2}}u^3(t) - (u^\star)^2\right| < (1-3B-\epsilon) (u^\star)^2,
\]
after
\[
t \geq T = T_1 + T_2 \log_2 \frac{1}{2\tau}.
\]

\end{proof}

\textbf{Proof of Theorem \ref{thm:gf-es}.}
Denote $\zeta = 100 \norm{\frac1n \X^\top\vec{\xi}}_\infty$. For $l\in S$, the gradient flow can be simplied as 
\begin{align*}
    \frac12 \frac{\partial u_l^2(t)}{\partial t}
    &=\frac2n \vec{w}_l^\top(t) \X_l^\top \vec{r}(t)\\
    &= 2 \vec{w}_l^\top(t)(\vec{w}^\star_l - \vec{w}_l(t)) 
    + \frac2n \vec{w}_l^\top \matrix{X}_l^\top \vec{\xi}\\
    &\geq 2 u_l^2(t) (u_l^\star)^2 \langle \vec{v}_l(t), \vec{v}_l^\star\rangle 
    - 2u_l^4(t) \norm{\vec{v}_l(t)}_2^2 - 2u_l^2(t) \norm{\vec{v}_l(t)}_2 \norm{\frac1n \X_l^\top \vec\xi}_2.
\end{align*}
Since the initialization is balanced $\frac12 u_l^2(0) = \norm{\vec{v}_l(0)}_2^2$, we know that from the balancing result Lemma~\ref{lemma:balancing},
\[
\frac12 u_l^2(t) = \norm{\vec{v}_l(t)}_2^2.
\]
Since the initialization of $\vec{v}_l(t)$ is aligned with direction $\frac1n\X_l^\top \vec{y}$, and with our assumption on orthogonal design, by Lemma~\ref{lemma:warm-up-init} and Lemma~\ref{lemma:dir-not-change}, if $\norm{\frac1n \X_l^\top \vec\xi}_2 \leq B (u_l^\star)^2$, we can further simplify the gradient flow as
\begin{align*}
\frac12 \frac{\partial u_l^2(t)}{\partial t} 
&\geq \sqrt2(1-2B^2)u_l^3(t) (u_l^\star)^2 - u_l^6(t) - \sqrt2u_l^3(t)B\\
&\geq \sqrt2(1-2B)u_l^3(t) (u_l^\star)^2 - u_l^6(t) - \sqrt2u_l^3(t)B,
\end{align*}
where the last inequality holds when $B<1$. We will verify that $B<1$ holds in the following analysis. 

If $\zeta \geq (u_{max}^\star)^2$, then our desired inequality is achieved at the initialization. 

If $(u_{min}^\star)^2 \leq \zeta \leq (u_{max}^\star)^2$, for these group that $\zeta \leq (u_l^\star)^2$, applying Lemma~\ref{lemma:gf-convergence} with 
\[
B = \frac{\norm{\frac1n \X_l^\top\vec{\xi}}_2}{(u_l^\star)^2} 
\leq \frac{\norm{\frac1n \X^\top\vec{\xi}}_\infty}{(u_l^\star)^2}
\leq \frac1{100},
\quad 
\tau = \frac{\epsilon}{(u_l^\star)^2}
\]
we obtain the convergence on magnitudes 
\[
|\norm{\vec{w}_l(t)}_2 - \norm{\vec{w}_l^\star}_2| \leq (3B+\epsilon)\norm{\vec{w}_l^\star}_2,
\]
after
\[
\frac{2^{1/3} (u_{l}^\star)^{4/3}}{\theta^2} \frac{1}{(1-6B)\sqrt{2} (u_{l}^\star)^2 \theta} +  \frac{2\log_2\frac{(u_{l})^2}{2\epsilon}}{3(u_{l}^\star)^2 (1/2 -3B)\left(\sqrt{2}(1/2-3B)(u_{l}^\star)^2 \right)^{1/3}}.
\]

If $\zeta \leq (u_{min}^\star)^2$, similarly applying Lemma~\ref{lemma:gf-convergence}, the number of iterations needed for entries on the support to converge is
\[
T_l =\frac{2^{1/3} (u_{max}^\star)^{4/3}}{\theta^2} \frac{1}{(1-6B)\sqrt{2} (u_{min}^\star)^2 \theta} +  \frac{2\log_2\frac{(u_{max})^2}{2\epsilon}}{3(u_{min}^\star)^2 (1/2 -3B)\left(\sqrt{2}(1/2-3B)(u_{min}^\star)^2 \right)^{1/3}}.
\]
We now have that for $l\in S$, 
\[
|\norm{\vec{w}_l(t)}_2 - \norm{\vec{w}_l^\star}_2| \leq (3B+\epsilon)\norm{\vec{w}_l^\star}_2,
\]
where $B = \frac{\norm{\frac1n\matrix{X}^\top \vec{y}}_\infty}{(u_{min}^\star)^2} \leq \frac1{100}, \forall l\in S$.

Recall that the direction is lower bounded by Lemma~\ref{lemma:warm-up-init} and Lemma~\ref{lemma:direction-bound-optimal},
\[
\left\langle \frac{\vec{w}_l(t)}{\norm{\vec{w}_l(t)}_2},
\frac{\vec{w}_l^\star}{\norm{\vec{w}_l^\star}_2} \right\rangle
\geq 1-B^2.
\]

Therefore, the error bound on the support is as follows,
\begin{align*}
    \norm{\vec{w}_l(t) - \vec{w}_l^\star}_\infty
    \leq \norm{\vec{w}_l(t) - \vec{w}_l^\star}_2
    &=     \norm{\left(\norm{\vec{w}_l(t)}_2 - (u_l^\star)^2\right)
    \frac{\vec{v}_l(t)}{\norm{\vec{v}_l(t)}}
    + (u_l^\star)^2 \left\langle \frac{\vec{v}_l(t)}{\norm{\vec{v}_l(t)}}, \vec{v}_l^\star\right\rangle}_2\\
    &\leq (3B+\tau) (u_l^\star)^2 + (u_l^\star)^2 \sqrt{2-2\left\langle \frac{\vec{v}_l(t)}{\norm{\vec{v}_l(t)}}, \vec{v}_l^\star\right\rangle}\\
    &= (3B+\tau) (u_l^\star)^2 + (u_l^\star)^2 \sqrt{2} B
    \leq \norm{\frac1n\matrix{X}^\top \vec{y}}_\infty + \epsilon.
\end{align*}
For $l \notin S$, we derive a lower bound on the growth rate
\begin{align*}
    \frac12 \frac{\partial u_l^2(t)}{\partial t}
    &=\frac2n \vec{w}_l^\top(t) \X_l^\top \vec{r}(t)\\
    &= 2 \norm{\vec{w}_l(t)}_2^2 
    + \frac2n \vec{w}_l^\top \matrix{X}_l^\top \vec{\xi}\\
    &\leq u_l^6(t) + \sqrt{2}u_l^4(t) B.
\end{align*}
By applying Lemma~\ref{lemma:gf-error} with $B = \norm{\frac1n\matrix{X}^\top \vec{y}}_\infty$, we obtain that before
\[
T_u = \frac{\log\frac1\theta}{2\theta^2 + \theta\sqrt2B}.
\]

Since $\theta < \frac{\epsilon}{2 (u_{max})^2}$, $T_l < T_u$ is ensured. 

\qed

\section{Analysis of gradient descent}
\label{sec:general-case}

\subsection{Monotonic updates}
\begin{lemma}
With an initialization $u(0)<u^\star$ and step size $\gamma \leq \frac{1}{4(u^\star)^2}$, the updating sequence 
\[
u(t) = u(t-1) + 2\gamma  u(t-1) [(u^\star)^2 - u^2(t-1)],
\]
is always bounded above by $u^\star$.
\label{lemma:monotonicity}
\end{lemma}
\begin{proof}
We prove it by contradiction. Assume there is a time $t$ s.t. 
\[
u(t) \leq  u^\star, u(t+1) >  u^\star.
\]
Therefore,
\[
u(t) + 2\gamma  u(t) [(u^\star)^2 - u^2(t)] > u^\star.
\]
Denote $\lambda = u(t)/u^\star$, we have that
\[
1+2\gamma(u^\star)^2(1-\lambda^2) - 1/\lambda > 0 
\]
for some $\lambda \in (0,1]$.

Let $f(\lambda) = 1+2\gamma(u^\star)^2(1-\lambda^2) - 1/\lambda$, we obtain the derivative 
\[
f'(\lambda) = -4\gamma (u^\star)^2 \lambda + \frac{1}{\lambda^2} > 0.
\]
However, $f_{max}(\lambda) = f(1) = 0$, and $f(\lambda) \leq 0$ for all $\lambda\in(0,1]$, which gives our desired contradiction.
\end{proof}

\subsection{Updates with bounded perturbations}
To study the general non-orthogonal and noisy case, we first extend the lemmas above to gradient dynamics with bounded perturbations. 

Consider the update on $\vec{v}(t)$ with bounded perturbations
\begin{equation}
\begin{aligned}
\vec{z}(t+1) &=  \vec{v}(t) + \eta_t u^2(t)((u^\star)^2 \vec{v}^\star - u^2(t)\vec{v}(t)) + \eta_t u^2(t) \vec{b}_t\\
  \vec{v}(t+1) &= \frac{\vec{z}(t+1)}{\norm{\vec{z}(t+1)}}.
  \label{eq:update-on-v}
\end{aligned}
\end{equation}
and the updates on $u(t)$
\begin{equation}
u(t+1) = u(t) + 2\gamma u(t) \vec{v}^\top(t+1) \{(u^\star)^2 \vec{v}^\star - u^2(t) \vec{v}(t+1)\} + 2\gamma u(t) e_t,
\label{eq:update-on-u}
\end{equation}

Note that if we choose $\eta_t = \frac1{u^4(t)}$, Eq.~\eqref{eq:update-on-v} is recast as
\begin{equation}
\begin{aligned}
\vec{z}(t+1) &=  \frac{(u^\star)^2}{u^2(t)} \vec{v}^\star + \frac{1}{u^2(t)}\vec{b}_t\\
  \vec{v}(t+1) &= \frac{\vec{z}(t+1)}{\norm{\vec{z}(t+1)}}.
  \label{eq:update-on-v-optimal}
\end{aligned}
\end{equation}


\begin{lemma}
Consider the update in Eq.~\eqref{eq:update-on-v-optimal}, if $\norm{\vec{b}_t} \leq B (u^\star)^2$ for some constant $0 < B < 1$, we have that 
\[
\langle \vec{v}(t+1), \vec{v}^\star \rangle \geq 1-B^2.
\]
\label{lemma:direction-bound-optimal}
\end{lemma}
\begin{proof}
We have that
\begin{align*}
    \langle \vec{z}(t+1), \vec{v}^\star \rangle 
    &= \frac{(u^\star)^2}{u^2(t)} + \frac{1}{u^2_l(t)}\langle \vec{b}_t, \vec{v}^\star \rangle \\
    \norm{\vec{z}(t+1)}^2
    &= \frac{(u^\star)^4}{u^4(t)} + 2\frac{(u^\star)^2}{u^4(t)}\langle \vec{b}_t, \vec{v}^\star \rangle + \frac{1}{u^4_l(t)}\norm{\vec{b}_t}^2,
\end{align*}
therefore, 
\begin{align*}
    \frac{\langle \vec{z}(t+1), \vec{v}^\star \rangle^2}{\norm{\vec{z}(t+1)}^2}
    &= \frac{\frac{(u^\star)^4}{u^4(t)} + 2\frac{(u^\star)^2}{u^4(t)}\langle \vec{b}_t, \vec{v}^\star \rangle + \frac{1}{u^4_l(t)}\langle \vec{b}_t, \vec{v}^\star \rangle^2 }{\frac{(u^\star)^4}{u^4(t)} + 2\frac{(u^\star)^2}{u^4(t)}\langle \vec{b}_t, \vec{v}^\star \rangle + \frac{1}{u^4_l(t)}\norm{\vec{b}_t}^2}\\
    &= 1 - \frac{\norm{\vec{b}_t}^2 - \langle \vec{b}_t, \vec{v}^\star \rangle^2 }{(u^\star)^4 + 2(u^\star)^2\langle \vec{b}_t, \vec{v}^\star \rangle + \norm{\vec{b}_t}^2} \\
    &= 1 - \frac{\norm{\vec{b}_t / (u^\star)^2}^2 - \langle \vec{b}_t /(u^\star)^2, \vec{v}^\star \rangle^2 }{1 + 2\langle \vec{b}_t/(u^\star)^2, \vec{v}^\star \rangle + \norm{\vec{b}_t/(u^\star)^2}^2}\\
    &= 1 - \frac{1 - \langle \vec{b}_t/\norm{\vec{b}_t}, \vec{v}^\star \rangle^2 }{1 + 2 \norm{\vec{b}_t}/(u^\star)^2 \langle \vec{b}_t/\norm{\vec{b}_t}, \vec{v}^\star \rangle + \norm{\vec{b}_t}^2/(u^\star)^4} \norm{\vec{b}_t/(u^\star)^2}^2\\
    &\geq 1-\norm{\vec{b}_t/(u^\star)^2}^2\\
    &\geq 1-B^2.
\end{align*}
Hence, we have that
\[
\langle\vec{v}(t+1),\vec{v}^\star\rangle \geq \sqrt{1-B^2} \geq 1-B^2.
\]
\end{proof}


\begin{lemma}
Consider the updates in Eq.~\eqref{eq:update-on-u} with $|e_t| \leq B$, if $u^2(0) \leq (u^\star)^2$, then $u^2(t) \leq (u^\star)^2 + B$ for all $t$. If $u^2(0) \geq (u^\star)^2$ and $|\langle \vec{v}(t), \vec{b}_t \rangle| \leq B_2\tau (u^\star)^2$, then $u^2(t) \geq (1-B_2)(u^\star)^2 - B$ for all $t$.
\label{lemma:bounded-updates}
\end{lemma}
\begin{proof}
Proof by contradiction similarly to Lemma~\ref{lemma:monotonicity}.
\end{proof}

\begin{lemma}
\label{lemma:phase1}
Fix the step size $\gamma$ for the update on $u(t)$, and choose $u(0)=\alpha \leq \frac15 u^\star$.
Consider the updates in Eq.~\eqref{eq:update-on-u} and Eq.~\eqref{eq:update-on-v} with 
$|\langle \vec{v}(t), \vec{b}_t \rangle| \leq \frac1{20} (u^\star)^2$  and $|e_t| \leq \frac1{20} (u^\star)^2$, then $T \geq \frac{\log \frac{(u^\star)^2}{2\alpha^2}}
{ 2\log (1+ \gamma \frac{1}{2}(u^\star)^2)}$, 
we have that  $u^2(T)\geq \frac12 (u^\star)^2$.
\end{lemma}
\begin{proof}
Apply Lemma~\ref{lemma:direction-bound-optimal} with $B=\frac1{20}$, 
\[
\langle \vec{v}(t+1), \vec{v}^\star \rangle
\geq 1-B^2
= 1 - \frac1{400} \geq \frac{4}{5}
\]

Starting from $t=1$, we have that 
\[
\vec{v}^\top(t) \{(u^\star)^2 \vec{v}^\star - u^2(t) \vec{v}(t)\} \geq \frac{4}{5}(u^\star)^2 - u^2(t), 
\]
therefore, we obtain an lower bound of the growth rate on $u(t)$, which reads
\begin{align*}
u(t+1) 
&\geq u(t) +2\gamma u(t) \left( \frac{4}{5}(u^\star)^2 - u^2(t) - \frac{1}{20}(u^\star)^2\ \right)\\
&= u(t) \left( 1+2\gamma 
\left( \frac34 (u^\star)^2 - u^2(t)\right) \right)\\
&\geq u(t) \left(1+ \gamma \frac{1}{2}(u^\star)^2 \right).\\
\end{align*}
Therefore, the requirement on the number of iterations is recast as
\begin{align*}
    &\alpha^2\left(1+ \gamma \frac{1}{2}(u^\star)^2 \right)^{2T} 
    \geq \frac12 (u^\star)^2\\
    \Longleftrightarrow
    &2T 
    \geq \frac{\log \frac{(u^\star)^2}{2\alpha^2}}
    { \log (1+ \gamma \frac{1}{2}(u^\star)^2)}
    \\
    \Longleftrightarrow
    &T \geq \frac{\log \frac{(u^\star)^2}{2\alpha^2}}
    { 2\log (1+ \gamma \frac{1}{2}(u^\star)^2)}.
\end{align*}

With these requirements, by Lemma~\ref{lemma:bounded-updates}, we also have that $u^2(t) \leq \frac32 (u^\star)^2, \forall t \geq0 $.
\end{proof}

\begin{lemma}
Fix the step size $\gamma$ for the update on $u(t)$, and choose the initialization $u(0)$ such that $ |(u^\star)^2 - u^2(0)| \leq \tau (u^\star)^2 $ where $0<\tau\leq 1/2$.
Consider the updates in Eq.~\eqref{eq:update-on-u} and Eq.~\eqref{eq:update-on-v} with 
$|\langle \vec{v}(t), \vec{b}_t \rangle| \leq \frac1{10}\tau (u^\star)^2$ and $|e_t| \leq \frac1{10}\tau (u^\star)^2$, then after $T
\geq \frac{5}{2\gamma(u^\star)^2 } $, 
we have that $\langle \vec{v}(t),\vec{v}^\star \rangle \geq 1-\frac15\tau^2$ for all $t\leq T$ and $|u^2(T) - (u^\star)^2|  \leq  \frac12 \tau (u^\star)^2$.
\label{lemma:phase2-convergence}
\end{lemma}
\begin{proof}
When $u^2(0) \leq (u^\star)^2$, by applying to Lemma~\ref{lemma:direction-bound-optimal}, we have that 
\[
\langle \vec{v}(t+1), \vec{v}^\star \rangle
\geq 1 - \left(\frac1{10}\tau\right)^2 \geq 1-\frac15 \tau^2,
\]
therefore,
\begin{align*}
u(t+1) 
&\geq u(t) +2\gamma u(t) \left( \left(1-\frac15\tau\right) (u^\star)^2 - u^2(t) - \frac{1}{10}\tau(u^\star)^2\ \right)\\
&= u(t) \left( 1+2\gamma 
\left(\left(1-\frac3{10}\tau\right) (u^\star)^2 - u^2(t)\right) \right).\\
\end{align*}

Further, we want to find an lower bound requirement on $T$ s.t.
\begin{align*}
    \left((u^\star)^2-\tau (u^\star)^2 \right) \left( 1+2\gamma 
\left(\left(1-\frac3{10}\tau\right) (u^\star)^2 - \left((u^\star)^2-\frac12\tau\right) (u^\star)^2\right) \right)^{2T} 
\geq (u^\star)^2-\frac12\tau (u^\star)^2 ,\\
\end{align*}

which can be relaxed as 
\begin{align*}
&\left((u^\star)^2-\tau (u^\star)^2 \right) \left( 1+\frac25\gamma T \tau
(u^\star)^2 \right)
\geq (u^\star)^2-\frac12\tau (u^\star)^2 \\
\Longleftrightarrow&
1+\frac25\gamma T \tau
(u^\star)^2
\geq \frac{(u^\star)^2-\frac12\tau (u^\star)^2}{(u^\star)^2-\tau (u^\star)^2}\\
\Longleftrightarrow&
\frac25 \gamma T\tau (u^\star)^2
\geq \frac{\frac12\tau (u^\star)^2}{((u^\star)^2-\tau (u^\star)^2) }\\
\Longleftrightarrow&
T
\geq \frac{5}{4\gamma (u^\star)^2(1-\tau) }\\
\Longrightarrow&
T
\geq \frac{5}{2\gamma(u^\star)^2 }.
\end{align*}

When $u^2(0) > (u^\star)^2$, we have that
\begin{align*}
    u(t+1) 
    &\leq u(t) + 2\gamma u(t)\left((u^\star)^2 - u^2(t) + \frac1{10}\tau(u^\star)^2 \right)\\
    &=  u(t) \left(1 + 2\gamma\left( \left( 1+\frac1{10}\tau\right)(u^\star)^2 - u^2(t)\right)\right). \\
    &\leq  u(t) \left(1 - \frac45\gamma\tau(u^\star)^2\right). \\
\end{align*}
Similarly, we want to get
\begin{align*}
&(u^\star)^2 + \frac12\tau (u^\star)^2
\geq  \left((u^\star)^2 + \tau (u^\star)^2 \right) \left( 1 - \frac45\gamma T \tau
(u^\star)^2 \right)\\
\Longleftrightarrow&
\frac{(u^\star)^2+\frac12\tau (u^\star)^2}{(u^\star)^2+\tau (u^\star)^2}
\geq 1-\frac45\gamma T \tau(u^\star)^2\\
\Longleftrightarrow&
\frac45 \gamma T\tau (u^\star)^2
\geq \frac{\frac12\tau (u^\star)^2}{(u^\star)^2+\tau (u^\star)^2 }\\
\Longleftrightarrow&
T
\geq \frac{5}{8\gamma (u^\star)^2(1+\tau) }\\
\Longrightarrow&
T
\geq \frac{5}{8\gamma(u^\star)^2 }.
\end{align*}

If $u(0) <= u^\star$ and $u(t) > u^\star, t<T$, or $u(0) > u^\star$ and $u(t) \leq u^\star, t<T$, we have already have $|u^2(t) - u^\star)^2|\leq \frac12\tau (u^\star)^2$. By Lemma~\ref{lemma:bounded-updates}, $|u^2(T) - u^\star)^2|\leq \frac12\tau (u^\star)^2$ remains to hold.

Hence, after $T \geq \frac5{2\gamma (u^\star)^2}$,
we have $|u^2(T) - u^\star)^2|\leq \frac12\tau (u^\star)^2$.
\end{proof}

\subsection{Analysis of perturbations}
We decompose the updates into several terms for later investigation.

The gradient of $\mathcal{L}(\cdot)$ on each $\vec{v}_l$ is
\begin{align*}
\frac{\partial\mathcal{L}}{\partial \vec{v}_l}
&=
-\frac1n u_l^2 \X_l^\top\left(\vec{y} - \sum_{l'\neq l} u^2_{l'}\X_{l'} \vec{v}_{l'}\right)
+ \frac1n  u^4_l \X_l^\top\X_l \vec{v}_l\\
& = -\frac1n u_l^2 \X_l^\top\left(\vec{y} - \sum_{l'=1}^L u^2_{l'}\X_{l'} \vec{v}_{l'}\right)
\end{align*}

When $l\in S$, the gradient update on each $\vec{v}_l$ is
\begin{align*}
\vec{z}_l(t+1)
&=
\vec{v}_l(t)  + \eta_{l,t} u_l^2(t) \frac1n \X_l^\top\left(\vec{y} - \sum_{l'=1}^L u^2_{l'}(t) \X_{l'} \vec{v}_{l'}(t) \right)\\
&= \vec{v}_l(t) + \eta_{l,t} u_l^2(t) ((u_l^\star)^2 \vec{v}_l^\star - u_l^2(t) \vec{v}_l(t))\\
&+ \eta_{l,t} u_l^2(t) \left( \frac1n \Xt_l\X_l - \matrix{I} \right) 
((u_l^\star)^2 \vec{v}_l^\star - u_l^2(t) \vec{v}_l(t))\\
&+ \eta_{l,t}u_l^2(t)\sum_{l'\neq l, l'\in S} \frac1n \X_l^\top \X_{l'}((u_{l'}^\star)^2 \vec{v}_{l'}^\star - u_{l'}^2(t) \vec{v}_{l'}(t))\\
&- \eta_{l,t}u_l^2(t)\sum_{l' \in S^c } \frac1n \X_l^\top \X_{l'}u_{l'}^2(t) \vec{v}_{l'}(t)\\
&+ \eta_{l,t}u_l^2(t) \frac1n \X_l^\top \xxi.
\end{align*}

The gradient of $\mathcal{L}(\cdot)$ on each $u_l$ is

\begin{align*}
\frac{\partial\mathcal{L}}{\partial u_l}
&=
-\frac2n u_l\left\langle \X_l\vec{v}_l,
\vec{y} - \sum_{l'\neq l} u^2_{l'}\X_{l'} \vec{v}_{l'}\right\rangle
+ \frac2n  u^3_l \norm{\X_l \vec{v}_l}^2\\
& = -\frac2n u_l\left\langle \X_l\vec{v}_l,
\vec{y} - \sum_{l'=1}^L u^2_{l'}\X_{l'} \vec{v}_{l'}\right\rangle
\end{align*}

When $l\in S$, the gradient update on $u_l$ reads

\begin{align*}
u_l(t+1)
&= u_l(t) + 
\gamma \frac2n u_l(t) \left\langle \X_l\vec{v}_l(t+1),
\vec{y} - \sum_{l'=1}^L u^2_{l'}(t)\X_{l'} \vec{v}_{l'}(t+1)\right\rangle\\
&=
u_l(t)+ 2\gamma u_l(t)  \vec{v}_l^\top(t+1) ((u_l^\star)^2\vec{v}_l^\star - u_l^2(t)\vec{v}_l(t+1))\\
&+ 
2\gamma u_l(t)  \vec{v}_l^\top(t+1) \left(\frac1n\X_l^\top \X_l - \matrix{I}\right) ((u_l^\star)^2\vec{v}_l^\star - u_l^2(t)\vec{v}_l(t+1))\\
&+2\gamma u_l(t) \vec{v}_l^\top(t+1) \frac1n \X_l^\top \sum_{l'\neq l, l'\in S}\X_{l'} ((u_{l'}^\star)^2 \vec{v}_{l'}^\star - u_{l'}^2(t) \vec{v}_{l'}(t+1))\\
&-2\gamma u_l(t) \vec{v}_l^\top(t+1) \frac1n \X_l^\top \sum_{l'\in S^c}\X_{l'} u_{l'}^2(t) \vec{v}_{l'}(t+1)\\
&+ 2\gamma u_l(t)\frac1n \vec{v}_l^\top(t+1) \X_l^\top \xxi.
\end{align*}

We now rewrite the definition of bounded perturbation in Eq.~(\ref{eq:update-on-v}, \ref{eq:update-on-u}), where the bounded perturbation $e_{l,t}$ on updates of $u_l(t)$ reads
\begin{align*}
e_{l,t} &=  \vec{v}_l^\top(t+1) \left(\frac1n\X_l^\top \X_l - \matrix{I}\right) ((u_l^\star)^2\vec{v}_l^\star - u_l^2(t)\vec{v}_l(t+1))\\
&+ \vec{v}_l^\top(t+1) \frac1n \X_l^\top \sum_{l'\neq l, l'\in S}\X_{l'} ((u_{l'}^\star)^2 \vec{v}_{l'}^\star - u_{l'}^2(t) \vec{v}_{l'}(t+1))\\
&- \vec{v}_l^\top(t+1) \frac1n \X_l^\top \sum_{l'\in S^c}\X_{l'} u_{l'}^2(t) \vec{v}_{l'}(t+1)\\
&+ \frac1n \vec{v}_l^\top(t+1) \X_l^\top \xxi,
\end{align*}
and the bounded perturbation $\vec{b}_{l,t}$ on updates of $\vec{v}_l(t)$ reads
\begin{align*}
\vec{b}_{l,t}
&= \left( \frac1n \Xt_l\X_l - \matrix{I} \right)
((u_l^\star)^2 \vec{v}_l^\star - u_l^2(t) \vec{v}_l(t))\\
&+ \sum_{l'\neq l, l'\in S} \frac1n \X_l^\top \X_{l'}((u_{l'}^\star)^2 \vec{v}_{l'}^\star - u_{l'}^2(t) \vec{v}_{l'}(t))\\
&- \sum_{l' \in S^c } \frac1n \X_l^\top \X_{l'}u_{l'}^2(t) \vec{v}_{l'}(t)\\
&+ \frac1n \X_l^\top \xxi.
\end{align*}

We show in Lemma~\ref{lemma:phase2-convergence} that when the perturbations are bounded, the direction is roughly accurate ($\langle\vec{v}_l(t),\vec{v}^\star\rangle$ is large) and $u_l(t)$ converges exponentially. Now we show below that when the direction is roughly accurate and $u_l(t)$ is close to $u_l^\star$, the perturbations are bounded.


\begin{lemma}
Assume  $\delta_{in}\leq \frac{ (u_{min}^\star)^2}{120(u_{max}^\star)^2}$ and $\delta_{out}\leq \frac{ (u_{min}^\star)^2}{120 s(u_{max}^\star)^2}$, $\alpha < \frac12 \sqrt{\frac{\tau_0}{L}} u_l^\star$, $\norm{\frac1n \X^\top\xxi}_\infty \leq \frac1{80}\tau_0(u_l^\star)^2$ and $|(u_l^\star)^2-u_l^2(0)|\leq\tau (u_l^\star)^2$ for each $l\in [L]$ where $0<\tau_0\leq\tau\leq1/2$. If $\langle\vec{v}_l(t), \vec{v}_l^\star\rangle \geq 1-\frac15 \tau^2$, then $|\langle \vec{v}_l(t),\vec{b}_{l,t}\rangle|\leq \frac1{10}\tau (u^\star_l)^2$ and $|e_{l,t}| \leq \frac1{10}\tau (u^\star_l)^2$.
\label{lemma:phase2-boundness}
\end{lemma}
\begin{proof}
We first verify 
\begin{align}
    \norm{(u_l^\star)^2 \vec{v}_l^\star - u_l^2(t) \vec{v}_l(t)}
    &=
    \norm{\{(u_l^\star)^2 - u_l^2(t)\} \vec{v}_l^\star - u_l^2(t) \{\vec{v}_l(t) - \vec{v}_l^\star\}}\notag\\
    &\leq
    |(u_l^\star)^2 - u_l^2(t)| + u_l^2(t) \norm{\vec{v}_l(t) - \vec{v}_l^\star}\notag\\
    &\leq \tau (u_l^\star)^2 + u_l^2(t) \sqrt{2-2\langle\vec{v}_l(t), \vec{v}_l^\star\rangle} \notag\\
    &\leq \tau (u_l^\star)^2 + \frac32 (u_l^\star)^2 \frac{\sqrt2}{\sqrt5} \tau \label{eq:bound-dir}\\
    &\leq 3\tau (u_l^\star)^2. \notag
\end{align}

By Assumption~\ref{assumptions}, we have that 
\begin{align*}
&\left|\vec{v}_l^\top(t) \left( \frac1n \Xt_l\X_l - \matrix{I} \right)
((u_l^\star)^2 \vec{v}_l^\star - u_l^2(t) \vec{v}_l(t))
+ \vec{v}_l^\top(t) \sum_{l'\neq l, l'\in S} \frac1n \X_l^\top \X_{l'}((u_{l'}^\star)^2 \vec{v}_{l'}^\star - u_{l'}^2(t) \vec{v}_{l'}(t))\right|\\
&\leq 3\delta_{in} \tau (u_{max}^\star)^2 + 3s\delta_{out} \tau (u_{max}^\star)^2 \leq \frac1{40}\tau (u_l^\star)^2 + \frac1{40}\tau (u_l^\star)^2 = \frac1{20}\tau (u_l^\star)^2.
\end{align*}
For the other two terms, we have that
\begin{align*}
    \left|\vec{v}_l^\top(t) \sum_{l' \in S^c } \frac1n \X_l^\top \X_{l'}u_{l'}^2(t) \vec{v}_{l'}(t) \right|
    \leq \delta (L-s) \alpha^2 \leq \frac1{80}\tau (u_l^\star)^2,
\end{align*}
and 
\begin{align*}
    \left|\vec{v}_l^\top(t) \frac1n \X_l^\top \xxi\right|
    &\leq 
    \norm{\vec{v}_l^\top(t)}_1 \norm{\frac1n \X_l^\top \xxi}_\infty\\
    &\leq 
    \norm{\vec{v}_l^\top(t)}_2 \norm{\frac1n \X_l^\top \xxi}_\infty\\
    &\leq \frac1{80}\tau (u_l^\star)^2.
\end{align*}

Therefore, 
\[
|e_{l,t}| = |\langle \vec{v}_{l}(t), \vec{b}_{l,t} \rangle|
\leq \frac1{20}\tau (u_l^\star)^2+\frac1{80}\tau (u_l^\star)^2+\frac1{80}\tau (u_l^\star)^2
\leq \frac1{10}\tau (u_l^\star)^2.
\]
\end{proof}

Lemma~\ref{lemma:phase2-convergence} shows that when the upper bound of perturbation is fixed, $u_l(t)$ grows. Now we show that after $u_l(t)$ grows, the upper bound of perturbations will be decreased.
\begin{lemma}
Assume  $\delta_{in}\leq \frac{ (u_{min}^\star)^2}{120(u_{max}^\star)^2}$ and $\delta_{out}\leq \frac{ (u_{min}^\star)^2}{120 s(u_{max}^\star)^2}$, $\alpha < \frac{\sqrt{\tau_0}}{2\sqrt{L}} u_l^\star$, $\norm{\frac1n \X^\top\xxi}_\infty \leq \frac1{80}\tau_0(u_l^\star)^2$ and $\langle\vec{v}_l(t), \vec{v}_l^\star\rangle \geq 1-\frac1{5} \tau^2$. If we achieve that $|(u_l^\star)^2-u_l^2(0)|\leq \frac12\tau (u_l^\star)^2$ for each $l\in [L]$ where $0<\tau_0\leq\tau\leq1/2$, then $\left|\langle \vec{v}_l(t), \vec{b}_{l,t} \rangle \right| \leq \frac{1}{20} \tau (u_l^\star)^2$ and $|e_{l,t}| \leq \frac{1}{20} \tau (u_l^\star)^2$.
\end{lemma}
\begin{proof}
Similarly to the proof of Lemma~\ref{lemma:phase2-convergence},

\begin{align*}
    \norm{(u_l^\star)^2 \vec{v}_l^\star - u_l^2(t) \vec{v}_l(t)}
    &\leq \frac12 \tau (u_l^\star)^2 + u_l^2(t) \sqrt{2-2\langle\vec{v}_l(t), \vec{v}_l^\star\rangle} \\
    &\leq \frac{1}2\tau (u_l^\star)^2 + \frac32 (u_l^\star)^2 \frac{1}{\sqrt5} \tau\\
    &\leq \frac{3}{2}\tau (u_l^\star)^2.
\end{align*}
By Assumption~\ref{assumptions}, we have that 
\begin{align*}
&\left|\vec{v}_l^\top(t) \left( \frac1n \Xt_l\X_l - \matrix{I} \right)
((u_l^\star)^2 \vec{v}_l^\star - u_l^2(t) \vec{v}_l(t))
+ \vec{v}_l^\top(t) \sum_{l'\neq l, l'\in S} \frac1n \X_l^\top \X_{l'}((u_{l'}^\star)^2 \vec{v}_{l'}^\star - u_{l'}^2(t) \vec{v}_{l'}(t))\right|\\
&\leq \frac32 \delta_{in} \tau (u_{max}^\star)^2  + \frac32 s\delta_{out} \tau (u_{max}^\star)^2 \leq \frac1{40}\tau (u_l^\star)^2,
\end{align*}
where $\delta \leq \frac1{60s}$.
Similarly, we obtain that
\[
|e_{l,t}| = |\langle \vec{v}_{l}(t), \vec{b}_{l,t} \rangle|
\leq \frac1{40}\tau (u_l^\star)^2+\frac1{80}\tau (u_l^\star)^2+\frac1{80}\tau (u_l^\star)^2
\leq \frac1{20}\tau (u_l^\star)^2.
\]
\end{proof}

By Lemma~\ref{lemma:phase1}, we know that after certain iterations, we have that $|u^2(t) - (u^\star)^2| \leq \frac12(u^\star)^2$. Starting from there, we will apply Lemma~\ref{lemma:phase2-convergence} and Lemma~\ref{lemma:phase2-boundness} iteratively until we have our desired accuracy. 

We just need to verify when $\tau = \frac12$, the condition of either Lemma~\ref{lemma:phase2-convergence} and Lemma~\ref{lemma:phase2-boundness} is satisfied. Note that the condition of Lemma~\ref{lemma:phase1} already satisfies the condition of Lemma~\ref{lemma:phase2-convergence} at $\tau = \frac12$. Note the condition of Lemma~\ref{lemma:phase1} is satisfied when $\delta_{in}\leq \frac{ (u_{min}^\star)^2}{120(u_{max}^\star)^2}$ and $\delta_{out}\leq \frac{ (u_{min}^\star)^2}{120 s(u_{max}^\star)^2}$, $\alpha \leq \frac14 (u^\star_{min})^2$, $\norm{\frac1n \X^\top\xxi}_\infty \leq \frac1{80}\tau_0(u_{min}^\star)^2$.

\subsection{Error analysis outside the support}
We only care about the growth rate of $u_l(t)$ when $l\notin S$.
When $l\in S^c$, the gradient updates on $u_l$ reads
\begin{align*}
u_l(t+1)
&= u_l(t) + 
\gamma \frac2n u_l(t) \left\langle \X_l\vec{v}_l(t),
\vec{y} - \sum_{l'=1}^L u^2_{l'}(t)\X_{l'} \vec{v}_{l'}(t)\right\rangle\\
&=
u_l(t) - 2\gamma u_l^3(t)\\
&-
2\gamma u^3_l(t)  \vec{v}_l^\top(t) \left(\frac1n\X_l^\top \X_l - \matrix{I}\right)\vec{v}_l(t)\\
&+2\gamma u_l(t) \vec{v}_l^\top(t) \frac1n \X_l^\top 
\sum_{l'\in S}\X_{l'} ((u_{l'}^\star)^2 \vec{v}_{l'}^\star - u_{l'}^2(t) \vec{v}_{l'}(t))\\
&-2\gamma u_l(t) \vec{v}_l^\top(t) \frac1n \X_l^\top \sum_{l'\neq l, l'\in S^c}\X_{l'} u_{l'}^2(t) \vec{v}_{l'}(t)\\
&+ 2\gamma u_l(t)\frac1n \vec{v}_l(t) \X_l^\top \xxi.
\end{align*}

Consider the initialization is $u_l(0) = \alpha$, we wonder the smallest number $t$ of iterations that we can ensure $u_l(t)\leq \sqrt{\alpha}$. Denote 
\begin{align*}
e_{l,t} &=  -u_l^2(t) - u^2_l(t)  \vec{v}_l^\top(t) \left(\frac1n\X_l^\top \X_l - \matrix{I}\right)\vec{v}_l(t)\\
&+ \vec{v}_l^\top(t) \frac1n \X_l^\top 
\sum_{l'\in S}\X_{l'} ((u_{l'}^\star)^2 \vec{v}_{l'}^\star - u_{l'}^2(t) \vec{v}_{l'}(t))\\
&- \vec{v}_l^\top(t) \frac1n \X_l^\top \sum_{l'\neq l, l'\in S^c}\X_{l'} u_{l'}^2(t) \vec{v}_{l'}(t)\\
&+\frac1n \vec{v}_l^\top(t) \X_l^\top \xxi.
\end{align*}

We have that 
\[
|e_{l,t}| \leq \alpha + \alpha\delta_{in} + \alpha \delta_{out} (L-s) + \frac32 (u^\star_{max})^2\delta_{out}s  + \norm{\frac1n \X_l^\top \xxi}_\infty.
\]
If $\alpha \leq \frac{1}{80L} (u_{min}^\star)^2$, $\delta_{in}\leq \frac{ (u_{min}^\star)^2}{120(u_{max}^\star)^2}$ and $\delta_{out}\leq \frac{ (u_{min}^\star)^2}{120 s(u_{max}^\star)^2}$, we have that
\begin{equation}
|e_{l,t}| \leq \frac{1}{20} (u_{min}^\star)^2 + \norm{\frac1n \X_l^\top \xxi}_\infty.
\label{eq:noise-bound}
\end{equation}

\begin{lemma}
Consider 
\[
u(t+1) = u(t)(1+2\gamma e_t) 
\]
where $|e_t|\leq B$ and $u(0)=\alpha$. Let the step size $\gamma \leq \frac{1}{4B}$, then for any $t\leq T = \frac1{32\gamma B} \log \frac1{\alpha^4}$, we have $u(t)\leq \sqrt{u(0)}$.
\label{lemma:noise-bound}
\end{lemma}
\begin{proof}
We start by observing,
\begin{align*}
    &\sqrt{\alpha} \geq u(t) \geq \alpha (1+2\gamma B)^t \\
    \Longleftrightarrow
    & t \leq \frac{\log \frac1{\sqrt{\alpha}}}{\log (1+2\gamma B)}.
\end{align*}
By using $\log x \leq x-1$, 
\begin{align*}
    \frac{\log \frac1{\sqrt{\alpha}}}{\log (1+2\gamma B)} \geq \frac1{2\gamma B} \log \frac1{\sqrt{\alpha}} 
    \geq \frac1{32\gamma B} \log \frac1{\alpha^4}. 
\end{align*}
\end{proof}

\section{Proof for Theorems in Section \ref{sec:gd-norm}}
\label{sec:proof-for-theorems}
\subsection{Proof of Theorem~\ref{thm:convergence-alg-1}}


\begin{proof}
If $\zeta \geq (u_{max}^\star)^2$, at the initialization, we already have for $\forall l\in [L]$
\begin{align*}
\norm{u_l^2(0) \vec{v}_l(0) - (u_l^\star)^2 \vec{v}_l^\star}_\infty
&\leq u_l^2(0) + (u_l^\star)^2
\leq \alpha^2 + (u_{\max}^\star)^2\\
&\leq 2 (u_{\max}^\star)^2 \leq 2\zeta \\
&\leq 160\norm{\frac1n \X^\top \xxi}_\infty \vee 160 \epsilon.
\end{align*}

If $\zeta \leq (u_{max}^\star)^2$, for those $l\in S$ such that $\zeta \leq (u_l^\star)^2 $,  we can apply Lemma~\ref{lemma:phase1}. After 
\[
T_1 = \frac{\log \frac{(u^\star_{l})^2}{2\alpha^2}}
{ 2\log (1+ \gamma \frac{1}{2}  (u^\star_{l})^2)},
\]
we obtain that $\frac12 (u_l^\star)^2 \leq u_l^2(T_1) \leq \frac32 (u_l^\star)^2$, where we also have that $\norm{\frac1n \X^\top \xxi}_\infty \leq \frac1{80}(u_{l}^\star)^2$ for every $l$.

Let $m_0$ be the number s.t. 
\[
 2^{-m_0-1}(u_{max}^\star)^2 \leq \zeta \leq  2^{-m_0}(u_{max}^\star)^2,
\]
which can be written as $m_0 = \lfloor \log_2 \frac{(u_{max}^\star)^2}\zeta\rfloor$. 
We can apply Lemma~\ref{lemma:phase2-convergence} and Lemma~\ref{lemma:phase2-boundness} together $m_0$ times. Then further after 
\[
T_2 =  \lfloor \log_2 \frac{(u_{max}^\star)^2}\zeta\rfloor \frac{5}{2\gamma(u^\star_{l})^2 },
\]
we have that
\begin{align*}
|u_l^2(T_2) - (u_l^\star)^2| 
&\leq 2^{-m_0}(u^\star_{max})^2 \leq 
2 \zeta\\ 
\langle \vec{v}_l(T_2), \vec{v}_l^\star\rangle 
&\geq 1-\frac15 2^{-2m_0}.
\end{align*}

Therefore,
\begin{equation}
\begin{aligned}
\norm{u_l^2(T_2) \vec{v}_l(T_2) - (u_l^\star)^2 \vec{v}_l^\star}_\infty 
&\leq \norm{u_l^2(T_2) \vec{v}_l(T_2) - (u_l^\star)^2 \vec{v}_l^\star}_2\\
&\leq 
\norm{(u_l^2(T_2)- (u_l^\star)^2)\vec{v}_l(T_2) - (u_l^\star)^2 (\vec{v}_l^\star-\vec{v}_l(T_2))}_2 \\
&\leq 
2^{-m_0}(u^\star_{max})^2 + (u_l^\star)^2 \sqrt{2-2\langle \vec{v}_l(T_2), \vec{v}_l^\star\rangle }\\
&\leq 
2^{-m_0}(u^\star_{max})^2  + (u_l^\star)^2 \frac25 2^{-m_0}\\
&\leq 2 \zeta.
\end{aligned}
\label{eq:gd-error-bound-on-support}
\end{equation}

Note that the above inequality holds for every $l\in S$ such that $(u_l^\star)^2 \geq \zeta$. For those $l$ such that $\zeta \geq (u_{l}^\star)^2$, we are not able to recover the true signal $(u_l^\star)^2$. the gradient dynamics on this group behaves as errors outside group, and bounded by  Lemma~\ref{lemma:noise-bound}.

For entries outside the support, we know that from Eq.~\eqref{eq:noise-bound},
\[
B = \frac{1}{20} (u_{min}^\star)^2 + \norm{\frac1n \X_l^\top \xxi}_\infty \leq \frac1{10} (\zeta \vee (u_{min}^\star)^2).
\]
By Lemma~\ref{lemma:noise-bound}, we have that before $T_3 \leq \frac1{32\gamma B}\log \frac1{\alpha^4}$, $u_l(T_3) \leq \sqrt{\alpha}$. 


When $\zeta\leq (u_{min}^\star)^2$, Eq.~\eqref{eq:gd-error-bound-on-support} holds for every $l\in S$. Therefore, a uniform number of iterations $T_1$ and $T_2$ for all groups is written as
\[
T_1 = \frac{\log \frac{(u^\star_{max})^2}{2\alpha^2}}
{ 2\log (1+ \gamma \frac{1}{2}  (\zeta \vee (u^\star_{min})^2))},
\]
and 
\[
T_2 =  \lfloor \log_2 \frac{(u_{max}^\star)^2}\zeta\rfloor \frac{5}{2\gamma(\zeta \vee (u^\star_{min})^2) }.
\]

All we left is to show that $T_3 \geq T_1+T_2.$ We observe that 
\begin{align*}
    T_1 =\frac{\log \frac{(u^\star_{max})^2}{2\alpha^2}}
{ 2\log (1+ \gamma \frac{1}{2}(\zeta \vee (u^\star_{min})^2))}
&\leq \frac{1+ \gamma \frac{1}{2}(\zeta \vee (u^\star_{min})^2))}{\gamma (\zeta \vee (u^\star_{min})^2))} \log \frac{(u^\star_{max})^2}{2\alpha^2} \\
&\leq \frac{2}{\gamma (\zeta \vee (u^\star_{min})^2)} \log \frac{(u^\star_{max})^2}{2\alpha^2}
\end{align*}
where the first inequality is by $\log x \geq \frac{x-1}x$. 


With our choice of small initialization on $\alpha$, we have  $T_1 \leq \frac12 T_3$, due to $\alpha < \frac{1}{(u_{max}^\star)^8}$. We have  $T_2 \leq \frac12 T_3$, because of $\alpha < \frac{\zeta^4}{(u_{max})^8}$.

Hence, we obtain that after $T_l = T_1+T_2 \geq \frac{\log \frac{(u^\star_{max})^2}{2\alpha^2}}
{ 2\log (1+ \gamma \frac{1}{2}(\zeta \vee (u^\star_{min})^2))} + \lfloor \log_2 \frac{(u_{max}^\star)^2}\zeta\rfloor \frac{5}{2\gamma (\zeta \vee (u_{min}^\star)^2)}$, and before $T_u=T_3 \leq \frac5{16\gamma (\zeta \vee (u_{min}^\star)^2)}\log \frac1{\alpha^4}$, 
\[
\norm{u_l^2(t)\vec{v}_l(t) - (u_l^\star)^2\vec{v}_l^\star}_\infty
 \lesssim
\begin{cases}
  \norm{\frac1n \X^\top \xxi}_\infty \vee  \epsilon, &\text{if }l\in S.\\
  \alpha, &\text{if }l\notin S.
\end{cases}
\]
\end{proof}


 \subsection{Proof for Corollary~\ref{cor:subgaussian-noise}}
Here is a standard result for sub-Gaussian noise. 
\begin{lemma}
  \label{lemma:bounding-max-noise}
  Let $\frac{1}{\sqrt{n}} \X$ be a $n \times p$ matrix with $\ell_2$-normalized columns. Let $\xxi \in \mathbb{R}^{n}$ be a vector of independent
  $\sigma^{2}$-sub-Gaussian random variables. Then, with probability at
  least $1 - \frac{1}{8p^{3}}$
  $$
  \norm{\frac1n \Xt \xxi}_\infty \lesssim \sqrt{\frac{\sigma^{2} \log p}{n}}.
  $$
\end{lemma}
\emph{Proof of Lemma \ref{lemma:bounding-max-noise}.}
Since the vector $\vec{\xxi}$ are made of independent $\sigma^2$-sub-Gaussian random variables and any column of $\X$ is $\ell_2$-normalized,
the random variable $\frac{1}{\sqrt{n}} (\Xt \vec{\xxi})_i$ is still $\sigma^2$-sub-Gaussian.

It is a standard result that for any $\epsilon > 0$,
\[
\mathbb{P}\left(\norm{\frac{1}{\sqrt{n}} \Xt \vec{\xxi}}_\infty > \epsilon\right)
\leq 
2p \exp\left(-\frac{\epsilon^2}{2\sigma^2}\right).
\]
Setting $\epsilon=2\sqrt{2\sigma^2\log(2p)}$, with probability at least $1-\frac{1}{8p^3}$ we have
\[
  \norm{\frac1n \Xt \vec{\xxi}}_\infty 
  \leq \frac{1}{\sqrt{n}}2\sqrt{\sigma^2\log(2p)}
  \lesssim \sqrt{\frac{\sigma^{2} \log p}{n}}.
\]
\qed

\textbf{Proof of Corollary \ref{cor:subgaussian-noise}.}
Since $\vec{\xxi}$ is made of independent $\sigma^2$-sub-Gaussian entries, by Lemma \ref{lemma:bounding-max-noise} with probability $1-1/(8p^3)$ we have
\[
\norm{\frac1n \Xt \vec{\xxi}}_\infty \leq 2 \sqrt{\frac{2\sigma^2 \log(2p)}{n}}.
\]
Hence, letting $\epsilon = 2 \sqrt{\frac{2\sigma^2 \log(2p)}{n}}$, we obtain that
\[
\norm{(\D\vec{u}(t))^2 \odot \vec{v}(t) - \wstar}_2^2 \lesssim 
\sum_{l\in S} \epsilon^2 + \sum_{l\notin S}\alpha
\leq s\epsilon^2 + (L-s) \frac{\epsilon^2}{L^2} \lesssim \frac{s\sigma^2 \log p}{n}.
\]
\qed

\subsection{Convergence for algorithm~\ref{alg:gd-norm-decrease}}

\begin{lemma}
Consider the update in Eq.~\eqref{eq:update-on-v}, choose the step size $\eta_t= \eta \leq \frac{4}{9(u^\star)^4}$, if $\langle \vec{v}(t) ,\vec{v}^\star \rangle \geq 1-\frac15 \tau$, $|u^2(t) - (u^\star)^2|\leq \tau (u^\star)^2$ and $\norm{\vec{b}_t} \leq \frac1{10}\tau (u^\star)^2$ for some constant $0 < \tau < \frac12$, we have that 
\[
\langle \vec{v}(t+1), \vec{v}^\star \rangle \geq 1-\frac15 \tau.
\]
\label{lemma:direction-bound}
\end{lemma}

\begin{proof}
We first rewrite $\vec{z}(t+1)$ as
\[
\vec{z}(t+1) = \eta u^2(t)(u^\star)^2 \vec{v}^\star + (1-\eta u^4(t))\vec{v}(t) + \eta u^2(t) \vec{b}_t.
\]
Therefore,
\begin{align*}
    \langle \vec{z}(t+1),\vec{v}^\star \rangle 
    &\geq \eta u^2(t)(u^\star)^2 + (1-\eta u^4(t))\langle\vec{v}(t),\vec{v}^\star\rangle + \eta u^2(t) \langle\vec{b}_t,\vec{v}^\star\rangle\\
    &\geq \eta u^2(t)(u^\star)^2 + (1-\eta u^4(t))\left( 1-\frac15\tau\right) - \eta u^2(t) \frac1{10}\tau (u^\star)^2\\
    \norm{\vec{z}(t+1)} 
    &\leq \eta u^2(t)(u^\star)^2 + (1-\eta u^4(t)) + \eta u^2(t) \frac1{10}\tau (u^\star)^2.
\end{align*}
We obtain that
\begin{align*}
    \langle \vec{v}(t+1),\vec{v}^\star\rangle
    &= \frac{\langle \vec{z}(t+1),\vec{v}^\star\rangle}{\norm{\vec{z}(t+1)}}
    \geq 1 - \frac{\frac15\tau (1-\eta u^4(t)) + 2\eta u^2(t) \frac1{10}\tau(u^\star)^2}{\eta u^2(t)(u^\star)^2 + (1-\eta u^4(t)) + \eta u^2(t) \frac1{10}\tau(u^\star)^2}\\
    &\geq 1 - \frac{1-\eta u^4(t) + \eta u^2(t)(u^\star)^2}{\eta u^2(t)(u^\star)^2 + (1-\eta u^4(t)) + \eta u^2(t) \frac1{10}\tau(u^\star)^2} \frac15\tau \\
    &\geq 1-\frac15\tau.
\end{align*}
\end{proof}

Note that compared with Lemma~\ref{lemma:direction-bound-optimal}, under the condition $\norm{\vec{b}_t} \leq B (u^\star)^2$, we get $\langle \vec{v}(t+1),\vec{v}^\star \rangle \geq 1-B$ instead of $\langle \vec{v}(t+1),\vec{v}^\star \rangle \geq 1-B^2$. Accordingly, we need to a new version for Lemma~\ref{lemma:phase2-boundness} with a smaller bound on $\delta$ to make up the loss in Lemma~\ref{lemma:direction-bound}.

\begin{lemma}
Assume  $\delta_{in}\leq \frac{ \sqrt{\tau_0}(u_{min}^\star)^2}{120 (u_{max}^\star)^2}$ and
$\delta_{out}\leq \frac{ \sqrt{\tau_0}(u_{min}^\star)^2}{120 s(u_{max}^\star)^2}$, $\alpha < \frac12 \sqrt{\frac{\tau_0}{L}} u_l^\star$, $\norm{\frac1n \X^\top\xxi}_\infty \leq \frac1{80}\tau_0(u_l^\star)^2$ and $|(u_l^\star)^2-u_l^2(0)|\leq\tau (u_l^\star)^2$ for each $l\in [L]$ where $0<\tau_0\leq\tau\leq1/2$. If $\langle\vec{v}_l(t), \vec{v}_l^\star\rangle \geq 1-\frac15 \tau$, then $|\langle \vec{v}_l(t),\vec{b}_{l,t}\rangle|\leq \frac1{10}\tau (u^\star_l)^2$ and $|e_{l,t}| \leq \frac1{10}\tau (u^\star_l)^2$.
\label{lemma:phase2-boundness-small}
\end{lemma}
\begin{proof}
Similarly to Lemma~\ref{lemma:phase2-boundness}, we have that
\begin{align}
    \norm{(u_l^\star)^2 \vec{v}_l^\star - u_l^2(t) \vec{v}_l(t)}
    &\leq \tau (u_l^\star)^2 + u_l^2(t) \sqrt{2-2\langle\vec{v}_l(t), \vec{v}_l^\star\rangle} \notag\\
    &\leq \tau (u_l^\star)^2 + \frac32 (u_l^\star)^2 \frac{\sqrt2}{\sqrt5} \sqrt\tau\\
    &\leq \left(1+2\frac1{\sqrt{\tau_0}}\right)\tau (u_l^\star)^2. \notag
\end{align}

By Assumption~\ref{assumptions}, we have that 
\begin{align*}
&\left|\vec{v}_l^\top(t) \left( \frac1n \Xt_l\X_l - \matrix{I} \right)
((u_l^\star)^2 \vec{v}_l^\star - u_l^2(t) \vec{v}_l(t))
+ \vec{v}_l^\top(t) \sum_{l'\neq l, l'\in S} \frac1n \X_l^\top \X_{l'}((u_{l'}^\star)^2 \vec{v}_{l'}^\star - u_{l'}^2(t) \vec{v}_{l'}(t))\right|\\
&\leq \left(1+2\frac1{\sqrt{\tau_0}}\right)\delta_{in} \tau (u_{max}^\star)^2
+\left(1+2\frac1{\sqrt{\tau_0}}\right)s\delta_{out} \tau (u_{max}^\star)^2
\leq \frac1{20}\tau (u_l^\star)^2,
\end{align*}
where $\delta \leq \frac{ \sqrt{\tau_0}(u_{min}^\star)^2}{60 s(u_{max}^\star)^2}$. The other two terms follows exactly what we did in Lemma~\ref{lemma:phase2-boundness}. Therefore, 
\[
|e_{l,t}| = |\langle \vec{v}_{l}(t), \vec{b}_{l,t} \rangle|
\leq \frac1{20}\tau (u_l^\star)^2+\frac1{80}\tau (u_l^\star)^2+\frac1{80}\tau (u_l^\star)^2
\leq \frac1{10}\tau (u_l^\star)^2.
\]
\end{proof}

\textbf{Proof to Theorem~\ref{thm:convergence-alg-2}.}
The proof is similar to that of Theorem~\ref{thm:convergence-alg-1}. For the first stage, we apply Lemma~\ref{lemma:phase1}, as nothing is changed from Theorem~\ref{thm:convergence-alg-1}. For the second stage, instead of applying    Lemma~\ref{lemma:phase2-convergence} and Lemma~\ref{lemma:phase2-boundness}, we apply Lemma~\ref{lemma:direction-bound} and Lemma~\ref{lemma:phase2-boundness-small} iteratively. To apply these lemmas, we first observe that

\begin{align*}
\zeta \leq \tau_0 (u_{max}^\star)^2 
\Longleftrightarrow 
\frac{\zeta}{(u_{max}^\star)^2} \leq \tau_0.
\end{align*}
Therefore the requirement on $\delta$'s becomes $\delta_{in}\leq \frac{ \sqrt{\tau_0}(u_{min}^\star)^2}{120 (u_{max}^\star)^3}$ and
$\delta_{out}\leq \frac{ \sqrt{\tau_0}(u_{min}^\star)^2}{120 s(u_{max}^\star)^3}$. The number of iterations and convergence results follow from the proof of Theorem~\ref{thm:convergence-alg-1}.

\qed

\textbf{The criterion for switching time.}
We provide some motivation for the practical criterion.
We first note that,  the criterion in Theorem~\ref{thm:convergence-alg-2} actually indicates a lower bound of switching time.
With more derivations, our results still hold if one choose to switch after the time when the criterion is first satisfied (instead of switching right at that time.)
Let us focus on the entries on the support.
In the proof of Theorem~\ref{thm:convergence-alg-1}, one can also obtain the convergence on $u_l(t)$ as the positiveness of $u_l(t)$ can be ensured with a small step size $\gamma$ (since the power-parametrization will recast the gradient updates into a multiplicative sequence).
Therefore, with an appropriate choice of $\tau$, the practical criterion
$\max\limits_{l\in S}\{{|u_l(t+1) - u_l(t)|}/{|u_l(t)+ \varepsilon|}\} < \tau$
would imply the theoretical criterion $u_l(t)^2 \ge \frac{1}{2} u_l^\star(t)^2$ on the support,
and therefore would indicate a possibly later switching time than what the theoretical criterion determines. For gradient updates outside the support,
we observe slow growth rate and hence the practical rule is likely satisfied on the non-support entry, which we observe in the numerical experiments. Note that the switching only happens when both the support and non-support entries fulfill the criterion.


\section{More numerical results}
\label{sec:more-numerical-res}

\subsection{Stability issue of Algorithm~\ref{alg:gd-norm} and standard GD}
\begin{figure}[ht!]
    \centering
    \subfloat[Numerical instability in direction estimations.]{
    \includegraphics[width=.9\linewidth]{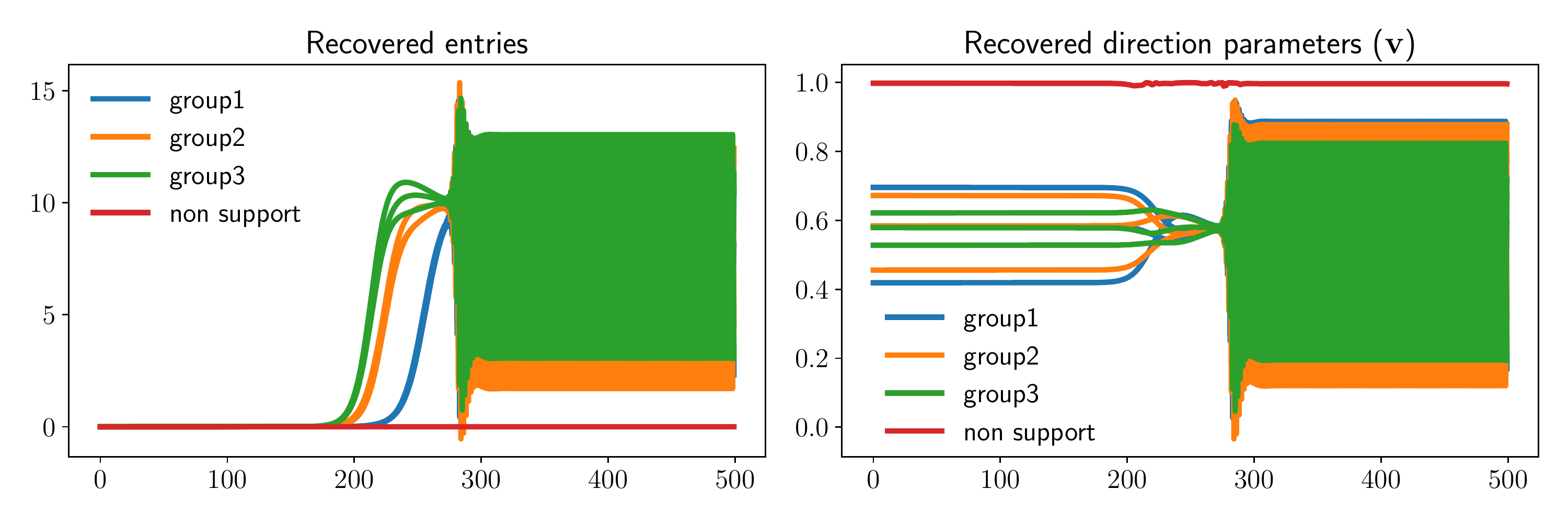}}
    
    \subfloat[Parameter estimation error remains small.]{\includegraphics[width=.9\linewidth]{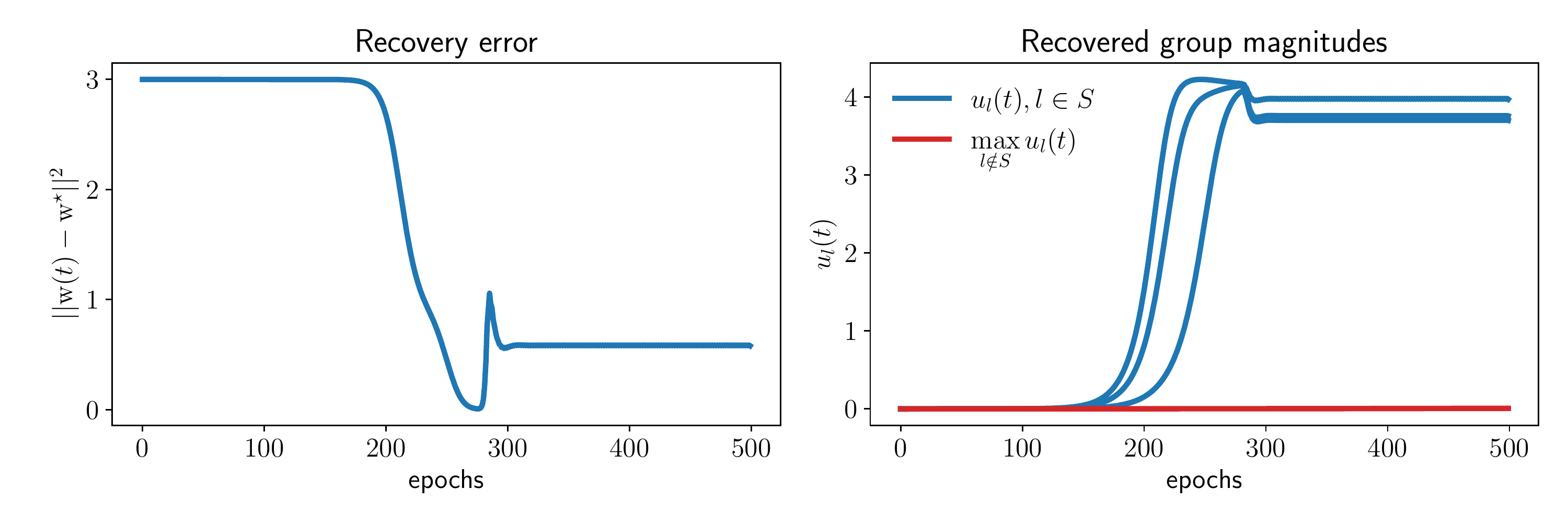}}
    \caption{Numerical instability of algorithm~\ref{alg:gd-norm}}
    \label{fig:instability_alg1}
\end{figure}
\textbf{Stability issue of Algorithm~\ref{alg:gd-norm}.} Figure~\ref{fig:instability_alg1} presents the recovered entries and direction parameters $\vec{v}(t)$ under the same setting as Figure~\ref{fig:convergence_alg1}. Because of the large learning rate on $\vec{v}$, the algorithm may not show a convergent result in the latter stage due to the irreducible error (perturbations). Although the parameter estimation is still reasonable with normalization on each $\vec{v}_l, l\in [L]$, we still aim to get a stable algorithm, which motivates our algorithm~\ref{alg:gd-norm-decrease}.

\begin{figure}[ht!]
    \centering
    \includegraphics[width=.9\linewidth]{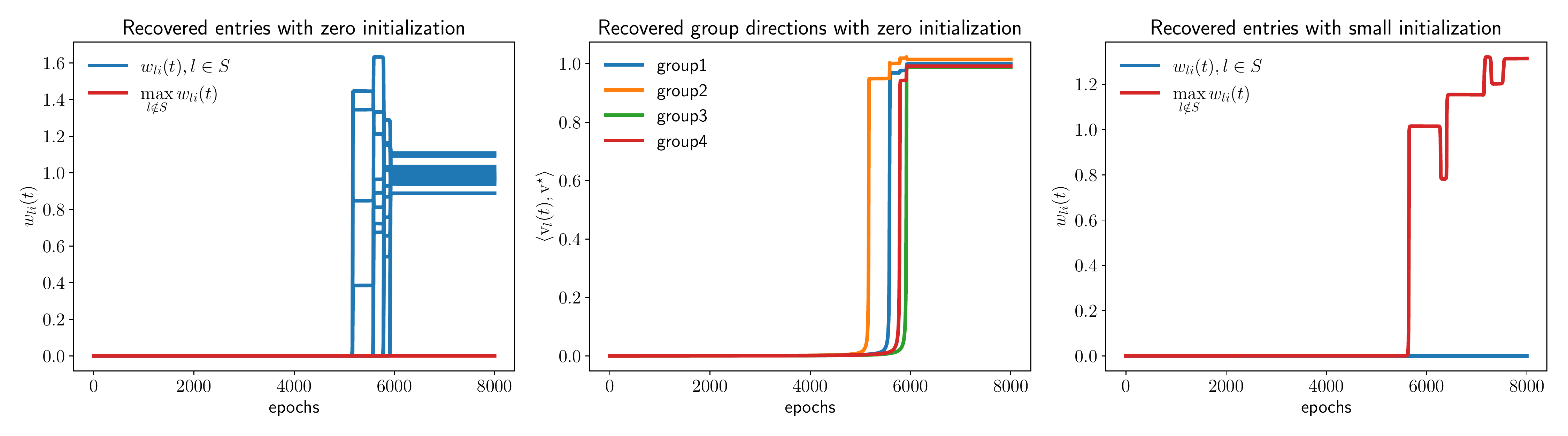}
    \caption{Gradient descent without weight normalization.}
    \label{fig:three_stages}
\end{figure}

\textbf{Standard gradient descent.}
To further understand how weight normalization affects the gradient dynamics, we conduct experiments using standard gradient descent without weight normalization. For that, we use the same setting as in Figure~\ref{fig:group_vs_sparse} and show the result in Figure~\ref{fig:three_stages}. 
The left and middle figures are based on zero initialization on $\vec{v}$. We see a numerically convergent result, and the inner product between learned and true directions starts to grow from $0$. As the directions guide the magnitude to grow, there is an extra stage for the directions to become roughly accurate. The choice of this initialization is necessary and subtle. The figure on the right is for small initialization $10^{-3}$, where the entries outside support get significant magnitudes, and the algorithm fails.

\subsection{Autoencoder with grouping layer}
The grouping layers have been used in grouped CNN and grouped attention mechanisms \citep{wu2021convolutional, xie2017aggregated, lee2018convolution}, which usually leads to parameter efficiency and better accuracy. To demonstrate the practical value of such grouping layers, we conduct the following experiment about learning good representations on MNIST.

\citep{jing2020implicit} proposed implicit rank-minimizing autoencoder (IRMAE), which is a deterministic autoencoder with implicit regularization. The idea is to apply more linear layers between encoder and decoder to penalize the rank of latent representation. A graphical illustration of the architecture is shown in Figure~\ref{fig:irmae}, where we explicitly show the last convolution layer and the linear layers in the latent space, which are absorbed into the last layer of the encoder in practice. This design is related to the power parametrization \citep{schwarz2021powerpropagation} trick to promote sparsity/low-rankness. One major advantage is that IRMAE produces a more interpretable latent representation, and the linear interpolation in the latent space gives a natural transition between two images.

\begin{figure}[ht!]
    \centering
    \includegraphics[width=\linewidth]{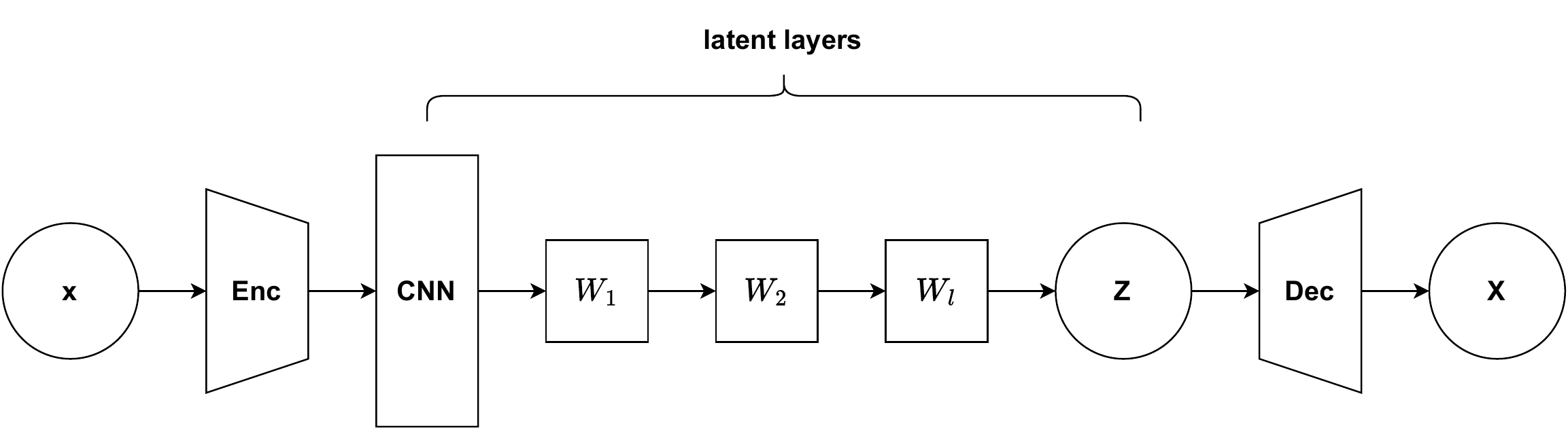}
    \caption{Implicit rank-minimizing autoencoder.}
    \label{fig:irmae}
\end{figure}
\begin{figure}[ht!]
    \centering
    \includegraphics[width=\linewidth]{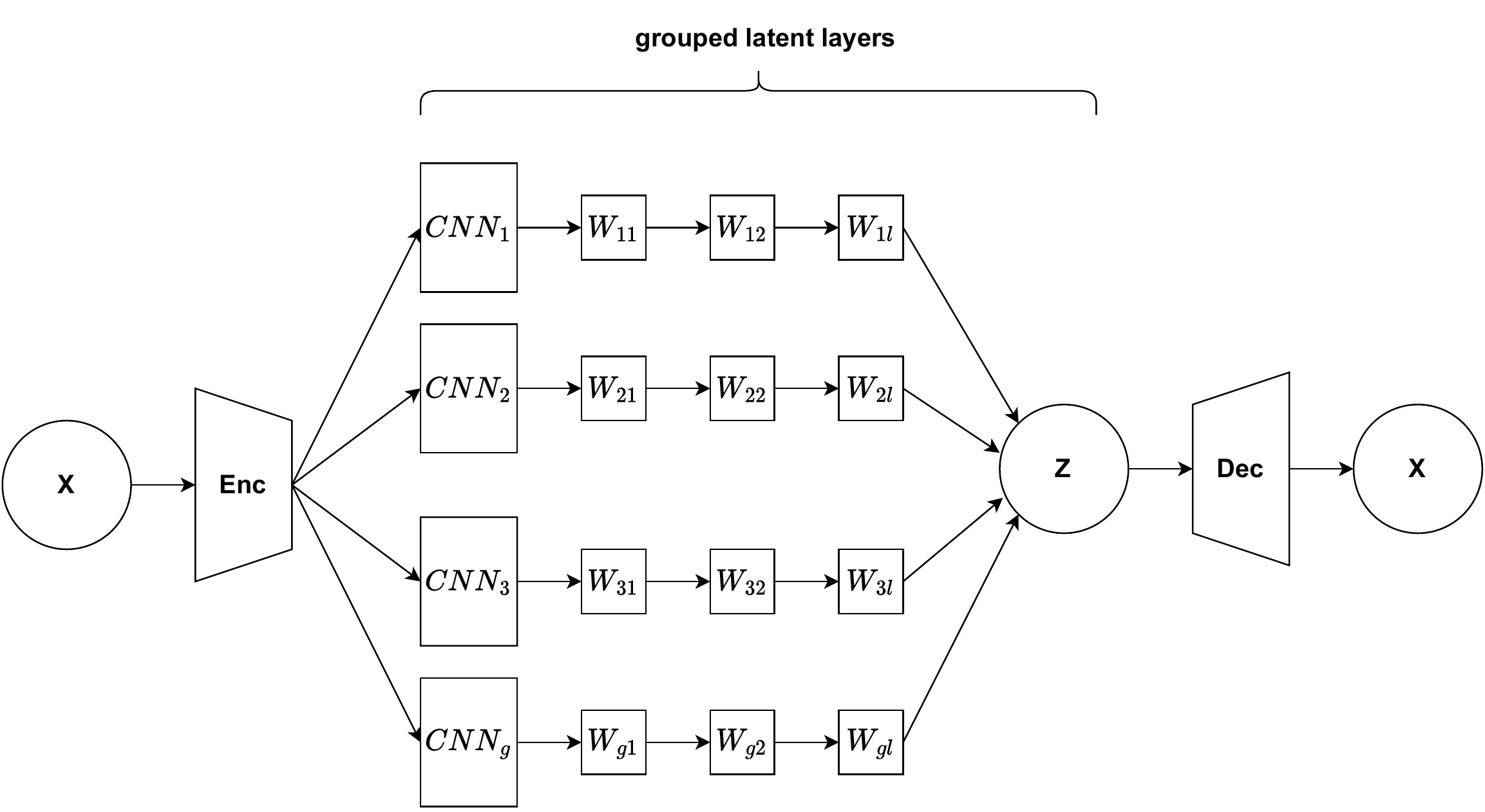}
    \caption{Implicit rank-minimizing autoencoder with grouping layers.}
    \label{fig:gae}
\end{figure}
 
 Inspired by our DGLNN, we design a CNN analog of it, which we call grouped autoencoder (GAE). The architecture is shown in Figure~\ref{fig:gae}. The channels feed into the last convolutional layer of encoder is separable into $g$ groups. The linear layers (power-parametrization) are applied within each group. Grouping channels of convolutional layers is a common practice to improve the parameter efficiency. With these grouping and power layers in the latent space, we expect it learns a better latent representation as IRMAE does.

\begin{figure}[ht!]
    \centering
    \begin{tabular}{@{}c@{ }c@{ }c@{ }}
    \rowname{AE}&\includegraphics[width=.8\linewidth]{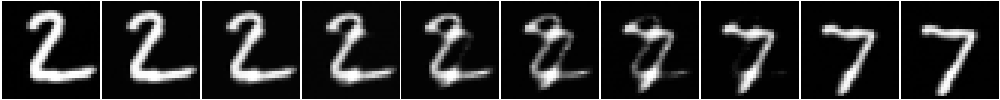}\\
    \rowname{VAE}&\includegraphics[width=.8\linewidth]{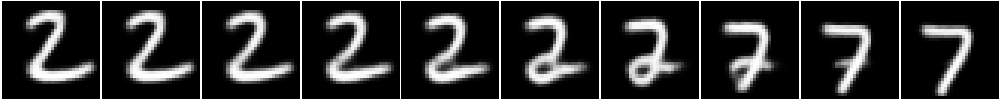}\\
    \rowname{GAE8} & \includegraphics[width=.8\linewidth]{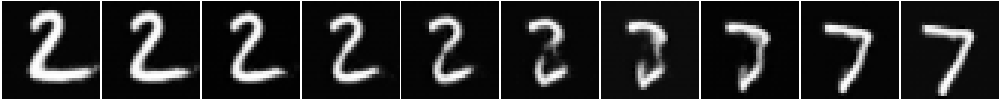}\\
    \rowname{GAE4} & \includegraphics[width=.8\linewidth]{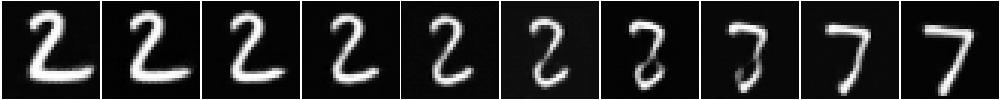}\\
    \rowname{IRMAE} & \includegraphics[width=.8\linewidth]{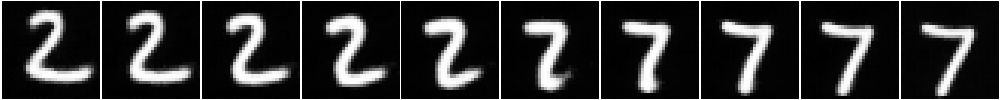}\\
    \end{tabular}
    \caption{Linear interpolations between data points on the MNIST dataset. GAE4/8 stands for grouped autoencoder with 4/8 groups.}
    \label{fig:ae_interpolation}
\end{figure}

The linear interpolations between data points in the latent space are shown in Figure~\ref{fig:ae_interpolation}. We compare the grouped autoencoder (GAE) with autoencoder (AE), variantional autoencoder (VAE) and implicit rank-minimizing autoencoder (IRMAE). We see that GAE outperforms AE and VAE, and gives comparable results with IRMAE. However, GAE achieves a better parameter efficiency as shown in Table~\ref{tab:num_params}.

\begin{table}[ht!]
    \centering
\begin{tabular}{ c|c } 
  & \# of params \\
  \hline
 IRMAE & 786K  \\ 
 GAE4 & 196K  \\ 
 GAE8 & 98K  \\ 
 \hline
\end{tabular}
    \caption{Number of parameters of hidden layers in latent space.}
    \label{tab:num_params}
\end{table}

\subsection{Experiments with gaussian measurements}

Besides the numerical results shown in Section~\ref{sec:simulation}, we conduct the following experiments with sampling each entry of $\matrix{X}$ from a standard normal distribution. 

\textbf{The effectiveness.} We follow the same setting with that Figure~\ref{fig:convergence_alg2} except changing Rademacher random variables to Gaussian random variables. The convergence of Algorithm~\ref{alg:gd-norm-decrease} is shown in Figure~\ref{fig:gaussian_alg2}. We see that the recovered entries, group magnitudes and directions successfully converge to the true ones.
\begin{figure}[ht!]
    \centering
    \includegraphics[width=\linewidth]{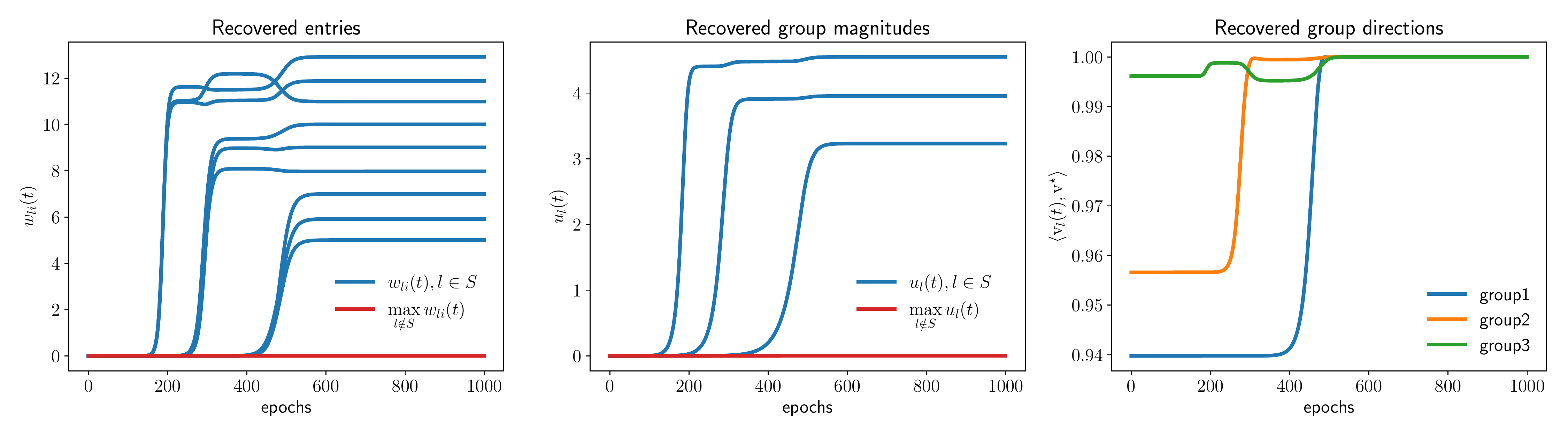}
    \caption{Convergence of algorithm 2 with Gaussian measurements}
    \label{fig:gaussian_alg2}
\end{figure}

\textbf{Comparisons with explicit regularization methods.}
We compare Algorithm~\ref{alg:gd-norm-decrease} with proximal gradient descent implemented in~\citep{carmichael2021yaglm} and primal-dual procedure~\citep{molinari2021iterative}. Each entry of $\matrix{X}$ is sampled from a standard Gaussian distribution. We set $n=150$ and $p=300$, and the number of non-zero entries is 10, divided into 3 groups with size 4. We vary the variance in the noise to achieve different signal-to-noise ratios (SNR). The experiment is repeated 30 times at each noise level. The average and standard deviation of the estimation error are depicted in Figure~\ref{fig:comparisons}. Our algorithm is consistently better than explicit regularization methods, whereas the primal-dual procedure has a comparable performance when SNR is large.

\begin{figure}[ht!]
    \centering
    \includegraphics[width=.5\linewidth]{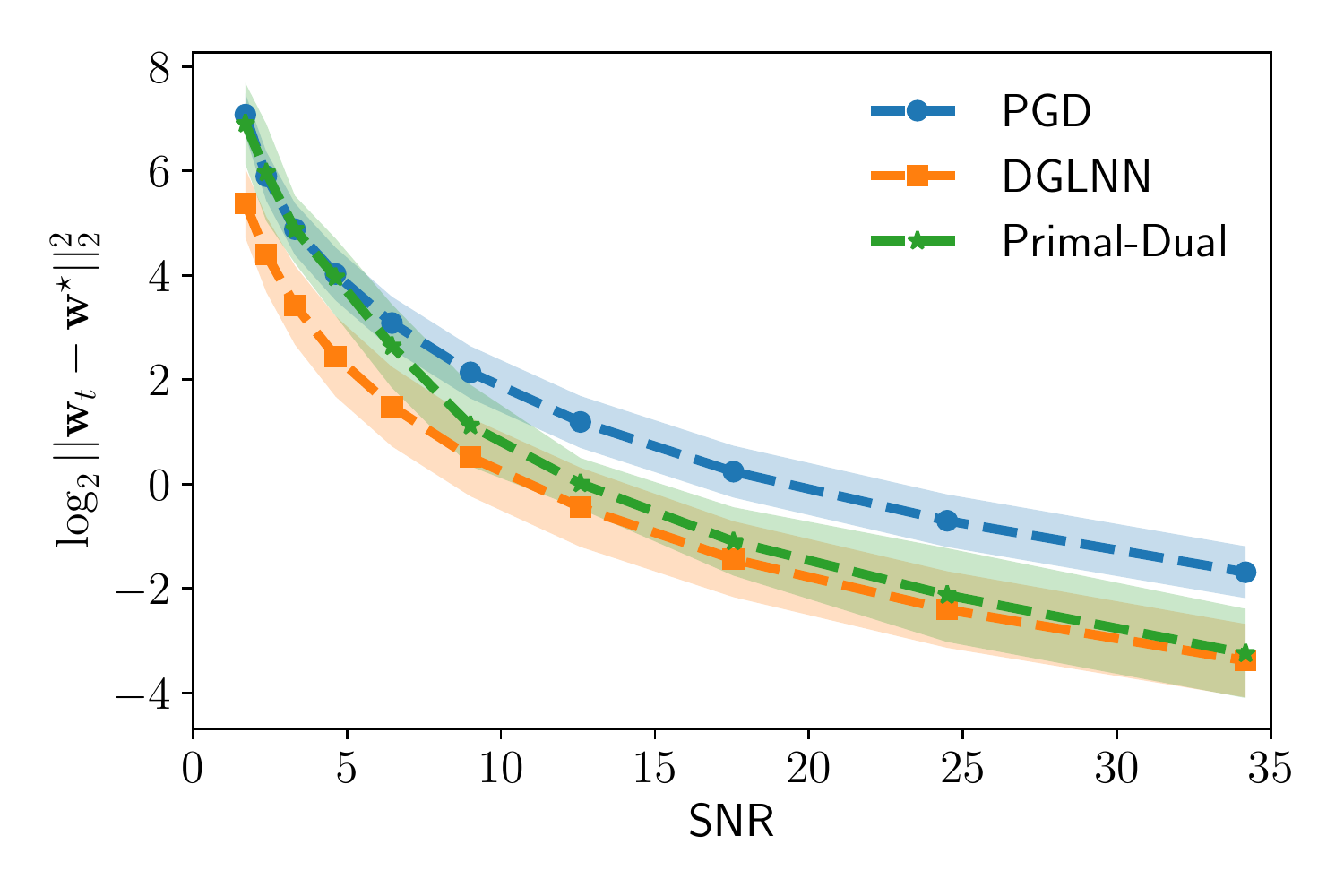}
    \caption{Comparisons with proximal gradient descent and iterative regularization.}
    \label{fig:comparisons}
\end{figure}

To further discover the potential applications of our findings, we use a gene expression dataset from the Microarray experiments of mammalian eye tissue samples \citep{scheetz2006regulation}. The dataset consists of 120 samples with 100 predictors that are expanded from 20 genes using 5 basis B-splines, as described in \citep{yang2015fast}. The goal is to predict the gene expression level of TRIM32, which causes Bardet-Biedl syndrome. We randomly split the data equally, and use the validation dataset for hyperparameter tuning and early stopping. We compare our approach with the commonly used proximal gradient descent and a primal-dual approach. The result is shown in Table~\ref{tab:real-data-res}. Our approach achieves the best performance among these three methods.

\begin{table}[ht!]
    \centering
\begin{tabular}{ c|c|c|c } 
  \hline 
  Test error &
  PGD & Primal-Dual & Our approach \\
  \hline
  MSE &
  0.03096 &
  0.02868 &
  0.02477 \\
 \hline
\end{tabular}
    \caption{Comparisons of MSE (mean squared error) on test set.}
    \label{tab:real-data-res}
\end{table}

\end{document}